\documentclass[letterpaper,journal]{IEEEtran}
\usepackage{amsmath,amsfonts}
\usepackage{amsthm}
\usepackage{algorithmic}
\usepackage{algorithm}
\usepackage{array}
\usepackage[caption=false,font=normalsize,labelfont=sf,textfont=sf]{subfig}
\usepackage{textcomp}
\usepackage{stfloats}
\usepackage{url}
\usepackage{verbatim}
\usepackage{graphicx}
\usepackage{cite}
\usepackage{multirow}
\usepackage{colortbl}
\usepackage{pifont}
\usepackage{booktabs}
\usepackage{xcolor}
\usepackage{makecell}
\usepackage[framemethod=TikZ]{mdframed}
\usepackage[colorlinks=true,citecolor=blue,linkcolor=blue,urlcolor=blue,hypertexnames=false]{hyperref}

\newcommand{\maketitlesupplementary}{%
  \section*{Supplementary Material}%
}

\newcommand{\figpanel}[2]{Fig.~\hyperref[#1]{\ref*{#1}(#2)}}

\newmdenv[
  backgroundcolor=gray!12,
  linecolor=gray!65,
  linewidth=1.2pt,
  roundcorner=4pt,
  skipabove=8pt,
  skipbelow=8pt,
  innertopmargin=8pt,
  innerbottommargin=8pt,
  innerleftmargin=9pt,
  innerrightmargin=9pt
]{theorembox}

\begin{document}

\title{Data-Aware Model Merging in the Essential Subspace for Multi-Task Knowledge Aggregation}

\author{Longhua Li,
        Lei Qi,
        Qi Tian,~\IEEEmembership{Fellow,~IEEE},
        Xin Geng,~\IEEEmembership{Senior Member,~IEEE}
\thanks{Longhua Li, Lei Qi and Xin Geng are with the School of Computer Science and Engineering, Southeast University, and Key Laboratory of New Generation Artificial Intelligence Technology and Its Interdisciplinary Applications (Southeast University), Ministry of Education, China, 211189 (e-mail: lhli@seu.edu.cn; qilei@seu.edu.cn; xgeng@seu.edu.cn).}
\thanks{Qi Tian is with Huawei Inc., Shenzhen 518129, China (e-mail: tian.qi1@huawei.com).}
\thanks{Corresponding authors: Lei Qi and Xin Geng.}
\thanks{This is an extended version of the paper presented at CVPR 2026 \cite{li2026model}.}
}


\maketitle

\begin{abstract}
Model merging enables training-free multi-task learning by aggregating knowledge from models fine-tuned from a shared checkpoint, but their task updates can interfere. Using proxy data, we find that task updates induce output shifts whose energy concentrates in a few dominant directions, while many low-energy directions can accumulate across tasks and impair performance. We call the span of the dominant directions the essential subspace. This observation motivates a data-aware framework to preserve essential task knowledge while reducing interference. Essential Subspace Decomposition (ESD) identifies each task's essential subspace from activation shifts measured on proxy data and represents its update compactly while preserving functional behavior. Essential Subspace Merging (ESM) weights each task-update direction by its corresponding output-shift eigenvalue, thereby prioritizing core task knowledge; it then orthogonalizes the weighted directions across tasks to reduce interference. ESM++ extends this merged model by dynamically integrating low-rank task-specific residual experts for adaptive knowledge aggregation. Experiments across diverse task sets and model scales show that ESM and ESM++ effectively aggregate task knowledge while mitigating inter-task interference. Code is available at \url{https://github.com/kiddo127/ESM}.
\end{abstract}


\IEEEpeerreviewmaketitle

\section{Introduction}
\label{sec:intro}

\IEEEPARstart{I}{n} recent years, the pre-training--fine-tuning paradigm has produced large numbers of task-specialized models adapted from the same pre-trained checkpoint. Model merging \cite{wortsman2022model,ilharcoediting,yan2025calm,sun2025towards,lin2026exme} aims to integrate the capabilities of these fine-tuned models into a single model without additional training, thereby providing a training-free route to multi-task learning. The central difficulty lies in composing multiple task-specific parameter updates without letting them interfere with one another.

Inter-task interference arises because each fine-tuned model encodes task knowledge as an update relative to the shared pre-trained model, and these updates may contain directions that are useful for one task but harmful or irrelevant to others. A straightforward averaging baseline, inspired by Model Soup \cite{wortsman2022model}, mixes all update directions and can therefore dilute useful task knowledge with noisy or conflicting components. To mitigate this issue, subsequent studies analyze task vectors, defined as the parameter differences between fine-tuned and pre-trained models \cite{ilharcoediting,yadav2023ties,matena2022merging,jindataless}. More recent methods apply Singular Value Decomposition (SVD) to task updates to identify low-rank structure and remove redundant parameter-space components \cite{stoicamodel,gargiulo2025task,marczakno,zhang2026dc}. However, parameter-space and functional importance can differ: as Equation~\ref{eq:svd_loss} shows, a direction with a small singular value may still induce large output changes. Thus, SVD-based truncation can discard functionally important updates and degrade task performance, motivating decompositions guided by activation effects.

\begin{figure}[t]
    \centering
    \includegraphics[width=\linewidth]{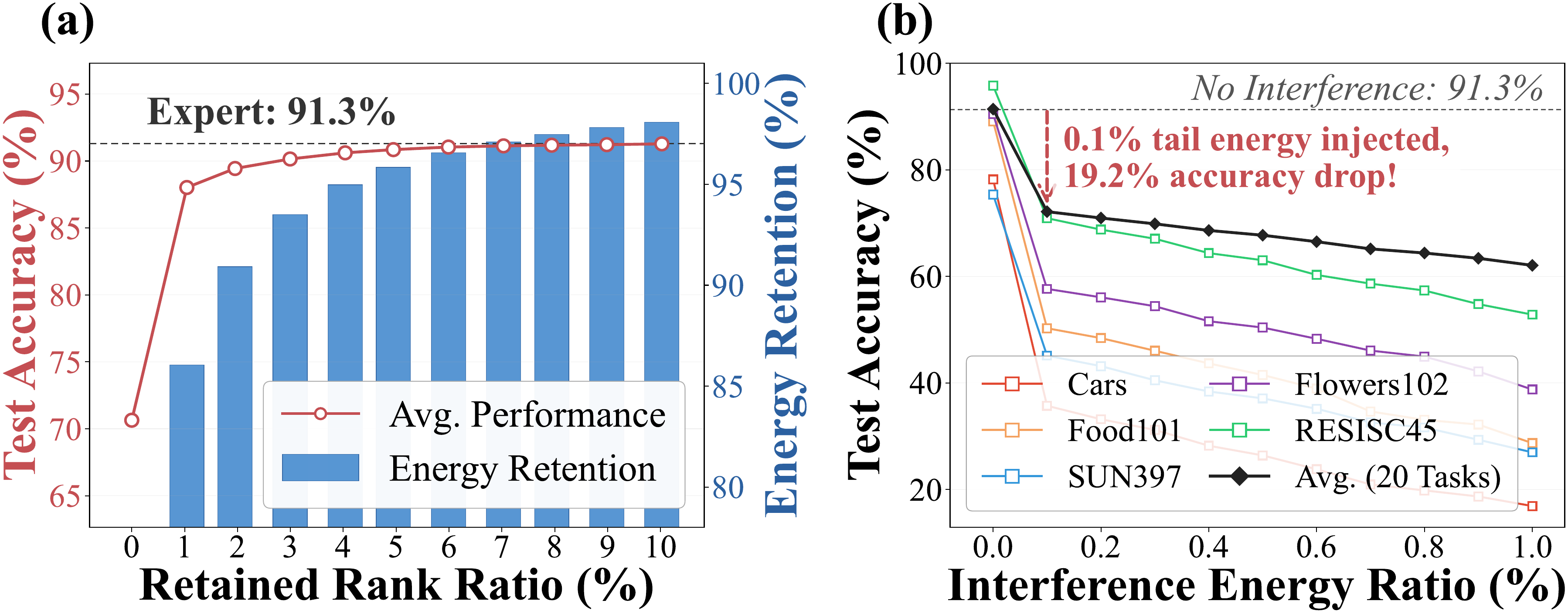}
    \caption{(a) Fine-tuned task updates are highly low-rank: retaining a small fraction of ranks preserves nearly all energy and approaches expert performance. The dual-axis plot uses \colorbox{red!10}{\textcolor{red!70!black}{red/left axis}} for task performance and \colorbox{blue!10}{\textcolor{blue!70!black}{blue/right axis}} for retained energy. (b) The x-axis denotes the energy fraction of low-energy tail components injected from the other 19 task updates into a target task, and the curves report the resulting change in target-task performance. Although each tail component carries little energy (e.g., only $0.1\%$), these components are largely useless for their own tasks but can accumulate across multiple tasks and substantially degrade other tasks.}
    \label{fig:insight}
\end{figure}


Motivated by this perspective, we revisit the low-rank phenomenon through the output activation shifts caused by task updates. Given a task update matrix, we perform uncentered Principal Component Analysis (PCA) on its induced output shifts and observe that the energy is highly concentrated in only a few principal directions, as shown in \figpanel{fig:insight}{a}.
We refer to the subspace spanned by these dominant directions as the \emph{essential subspace}, since it captures the functional directions most responsible for the task behavior. Conversely, most remaining directions contain very little activation-shift energy and contribute little to the task itself. Nevertheless, these low-energy directions are not harmless in model merging: when accumulated across many task updates, they can introduce substantial cross-task interference and degrade all tasks simultaneously, as illustrated in \figpanel{fig:insight}{b}.

We build on our preliminary work, Essential Subspace Merging (ESM) \cite{li2026model}, which introduced activation-shift-aware static merging for ViT models. ESM first applies Essential Subspace Decomposition (ESD): uncentered PCA on each task update's output shifts defines an essential subspace, onto which the update is projected and truncated to its dominant components. It then concatenates the retained factors across tasks and orthogonalizes them before fusing them into one model. In this extended version, we retain the ESM name but refine this composition step by weighting parameter directions according to their output-shift eigenvalues before orthogonalization, prioritizing functionally important directions while reducing inter-task conflict. ESM remains static and training-free, without additional parameter optimization.

We further extend ESM to ESM++, a dynamic, training-free merging framework that goes beyond static ViT model fusion. Starting from the statically merged ESM model, ESM++ decomposes the residual between each fine-tuned expert and the ESM base model, stores the resulting low-rank components as task-specific experts, and dynamically selects an expert at each layer through prototype-based routing. This extension preserves the compact shared representation learned by ESM while recovering task-specific specialization through adaptive composition. Moreover, we broaden essential subspace merging beyond ViT architectures to vision models, discriminative language models, and generative language models.

Our main contributions are summarized as follows:
\begin{itemize}
    \item We show that task-update-induced output shifts concentrate in a few essential directions, whereas low-energy residual directions accumulate and cause inter-task interference. Based on this insight, we propose Essential Subspace Decomposition (ESD), which optimally preserves functional behavior under low-rank truncation.
    \item We propose ESM, a static training-free merging method that decomposes task updates with ESD, removes non-essential directions, applies eigenvalue-based weighting to emphasize high-energy components, and orthogonalizes them into a compact multi-task model.
    \item We further extend ESM to ESM++, a dynamic training-free merging method that preserves task-specific residual knowledge as low-rank experts and composes them at inference time using prototype-based routing.
    \item We conduct extensive experiments on vision models, discriminative and generative language models across multiple task sets and model scales, demonstrating state-of-the-art performance in multi-task model merging.
\end{itemize}
\section{Related Work}
\label{sec:related_work}

\subsection{Model Merging}
Model merging combines task-specific models into a unified model without retraining. Directly fusing models from independent runs can fail when they lie in different loss basins; alignment methods address this through training-time \cite{li2026improving} or post-training procedures \cite{singh2020model,tatro2020optimizing,pena2023re,chen2024privfusion}. Most recent methods instead merge models fine-tuned from a shared checkpoint. Model Soup \cite{wortsman2022model} averages their weights, while Task Arithmetic \cite{ilharcoediting} represents task knowledge as parameter-difference vectors.

However, directly averaging task vectors can cause severe interference when updates conflict. TIES-Merging \cite{yadav2023ties} addresses this by trimming redundant parameters and resolving sign conflicts before aggregation; AdaMerging \cite{yangadamerging} learns adaptive task-wise coefficients, while DARE \cite{yu2024language} drops redundant update parameters and rescales those retained. Fisher Merging \cite{matena2022merging} and RegMean \cite{jindataless} instead use Fisher information or input statistics to weight parameter updates. Beyond these approaches, some methods refine aggregation at the parameter or layer level \cite{du2024parameter,zhang2024knowledge,wanglines,chen2026rethinking}, while others use modular representations \cite{chengwhoever,huang2024emr} or retain personalized models alongside global aggregation, as in MAP for incomplete-class settings \cite{li2024map}.

Spectral methods exploit low-rank structure in task updates: TSV-M \cite{gargiulo2025task} merges SVD components, Iso-CTS \cite{marczakno} constructs an isotropic common subspace with task-specific refinements, and DC-Merge \cite{zhang2026dc} smooths singular values to promote directional consistency. These methods derive importance from parameter-space structure rather than the functional effects of updates on inputs. In contrast, our ESD derives its subspace from task-specific output shifts. For a fixed rank, ESD minimizes expected output truncation error relative to parameter-space SVD. We use ESD in two complementary methods: static ESM and dynamically routed ESM++.

\subsection{Low-Rank Decomposition of Model Weights}
Decomposing model weights has been extensively studied in various areas \cite{wangsvd,li2026stratified,li2026energy}. One of the most popular approaches is based on the low-rank assumption of weight matrices.
The LoRA family of methods \cite{hulora,dettmers2023qlora,ding2023parameter} assumes that fine-tuning updates are inherently low-rank and learns compact matrices to parameterize these updates.
Other methods use SVD-based decompositions for parameter-efficient fine-tuning \cite{sun2022singular,han2023svdiff} or model compression \cite{li2023losparse,saha2023matrix,wangsvd,wang2025svd,zhang2026dc}. More recently, low-rank decompositions have been applied to model merging \cite{stoicamodel,gargiulo2025task,marczakno}, combining task updates in reduced subspaces to mitigate inter-task interference.
In contrast, our proposed Essential Subspace Decomposition (ESD) constructs the decomposition space not from the weight updates themselves but from the activation shifts induced by these updates. By capturing task-specific principal directions in activation space, ESD produces compact yet expressive task representations, reducing cross-task interference while preserving task fidelity.

\section{Methodology}

\subsection{Preliminaries on Model Merging}
Model merging integrates a collection of task-specific models, each fine-tuned from a common pre-trained checkpoint, into a single unified model without additional retraining. Let $W_0$ and $W_t$ denote the collections of weight matrices in the pre-trained model and the expert fine-tuned on task $t$, respectively, where $t=1,\ldots,T$. The fundamental object of interest is the task update, which captures how fine-tuning shifts the model away from the pre-trained weights.
Following Task Arithmetic \cite{ilharcoediting}, the task vector for task $t$ is defined as
\begin{equation}
\tau_{t}=\operatorname{Flatten}(W_t-W_0).
\end{equation}
Given the structured nature of models, it is preferable to retain the matrix form of the update rather than flattening it into a vector. The task matrix of layer $\ell$ is defined as
\begin{equation}
\Delta W_t^{ \left( \ell \right)}= W_t^{ \left( \ell \right)}- W_0^{ \left( \ell \right)}.
\end{equation}
Each $\Delta W_t^{(\ell)}$ represents the task-specific update at layer $\ell$, preserving the row--column structure essential for spectral analysis. The goal of model merging is to construct merged weights $W_{\text{merge}}$ that support all tasks, typically in the form:
\begin{equation}
W_{\text{merge}}= W_0+f \left( \Delta W_1, \cdots, \Delta W_T \right),
\end{equation}
The design of the merging function $f(\cdot)$ is a central focus of model merging research \cite{ilharcoediting,stoicamodel,gargiulo2025task,marczakno}.

\subsection{Essential Subspace Decomposition}
\label{sec:ESD}

The motivation of ESM is that task knowledge is typically concentrated in a few functional directions, whereas numerous weak directions can accumulate across tasks and become a major source of interference. Therefore, before composing task updates, we decompose each task matrix, retain only the directions that are essential to its functional behavior, and then orthogonalize or route the retained components across tasks. Unlike previous methods \cite{gargiulo2025task,marczakno} that merge models in truncated singular-vector subspaces obtained via SVD, we decompose and merge task matrices in a subspace aligned with their induced output shifts.
For simplicity, unless otherwise specified, we omit the layer index $\ell$ and task identifier $t$ in the task matrix and denote it simply as $\Delta W$.

\subsubsection{\textbf{Limitations of Direct Task Matrix Decomposition}}
Recent model merging methods often directly decompose the task matrix $\Delta W \in \mathbb{R}^{d_{\text{out}} \times d_{\text{in}}}$ with SVD, keeping only the top-$r$ singular components to retain dominant parameter-space directions and reduce task interference, yielding the truncated approximation $\widehat{\Delta W}$.
This strategy is reasonable because it removes many small components that are unlikely to help the current task but may still interfere with other tasks after aggregation. However, the criterion is parameter-centric: it minimizes the Frobenius norm reconstruction error of $\Delta W$ without considering the input feature distribution. Thus, a direction with small singular value is not guaranteed to be functionally unimportant. The expected output error under rank-$r$ SVD truncation for $x\sim\mathcal{D}$ is given by:
\begin{equation}\label{eq:svd_loss}
\mathbb{E}_{x\sim\mathcal{D}} \left[ \|\Delta W x - \widehat{\Delta W} x \|_2^2 \right] = \sum_{i=r+1}^s\sigma_{i}^{2} \cdot \mathbb{E}_{x\sim\mathcal{D}} \left[ (v_{i}^{\top}x)^2 \right],
\end{equation}
where $s$ denotes the number of non-zero singular values and $\{u_i\}$, $\{v_i\}$ are the left and right singular vectors. The proof of this SVD truncation loss is provided in Appendix~\ref{sec:Proof_SVD_Truncation_Error}.

As shown, the error depends not only on the discarded singular values $\sigma_i$, but also on the alignment between the input distribution and the right singular vectors $v_i$. A direction with small $\sigma_i$ may be functionally critical if inputs project strongly onto $v_i$. Conversely, retaining parameter-dominant directions that have little activation effect may introduce unnecessary cross-task overlap. By ignoring the input distribution, SVD may both discard functionally essential information and retain directions that mainly contribute to interference.

\subsubsection{\textbf{Output Shift-Aware Decomposition}}
To address this limitation, we introduce Essential Subspace Decomposition (ESD), which constructs a basis from the principal directions of output shifts induced by the task update matrix $\Delta W$. Instead of asking which parameter directions reconstruct $\Delta W$ most accurately, ESD asks which output directions explain the functional change caused by $\Delta W$ on representative inputs. This directly connects decomposition to task behavior.

For each task $t$, we sample a small unlabeled proxy dataset. By performing a forward pass through the task-specific fine-tuned model and recording the layer-wise input features, we obtain an input matrix for each layer. Specifically, let $X_{\text{proxy}} \in \mathbb{R}^{n \times d_{\text{in}}}$ contain $n$ recorded input tokens. The output shifts attributable to the weight update at that layer are
\begin{equation}\label{eq:activation_shift}
\Delta O = X_{\text{proxy}} \Delta W^{\top} \in \mathbb{R}^{n \times d_{\text{out}}},
\end{equation}
which capture the functional footprint of $\Delta W$ on the proxy inputs. We perform uncentered PCA without subtracting the mean, thereby retaining the mean output shift as part of the task-specific signal. Equivalently, we eigendecompose the empirical second-moment matrix
\begin{equation}\label{eq:activation_second_moment}
M_{\Delta O}=\frac{1}{n}\Delta O^{\top}\Delta O.
\end{equation}
The resulting eigenvectors $\boldsymbol{e}_i$ and eigenvalues $\lambda_i$ are ordered by decreasing output-shift energy. The eigenvectors form an orthonormal basis $E = [\boldsymbol{e}_1, \boldsymbol{e}_2, \dots, \boldsymbol{e}_{d_{\text{out}}}] \in \mathbb{R}^{d_{\text{out}} \times d_{\text{out}}}$ for the output space.
The original task matrix $\Delta W$ is projected onto $E$, yielding the coordinate matrix $C = E^{\top} \Delta W\in\mathbb{R}^{d_{\text{out}}\times d_{\text{in}}}$, and can be factorized as
\begin{equation}\label{eq:esd_factorization}
\Delta W = EC = E(E^{\top} \Delta W).
\end{equation}
We truncate to the top-$r$ principal components to form the essential basis $\hat{E}=[e_1, \dots, e_r] \in \mathbb{R}^{d_{\text{out}} \times r}$. The corresponding coordinate matrix is $\hat{C} = \hat{E}^{\top} \Delta W\in\mathbb{R}^{r\times d_{\text{in}}}$, leading to the low-rank approximation:
\begin{equation}\label{eq:esd_approximation}
\widehat{\Delta W} = \hat{E} \hat{C} = \hat{E} (\hat{E}^{\top} \Delta W).
\end{equation}
Let $\mathcal{D}_{\text{proxy}}$ denote the empirical distribution of the recorded proxy features. Under this decomposition, the output truncation error on that distribution is given by:
\begin{equation}\label{eq:esd_loss}
\mathbb{E}_{x\sim\mathcal{D}_{\text{proxy}}} \left[ \|\Delta W x - \widehat{\Delta W} x \|_2^2 \right] = \sum_{i=r+1}^{d_{\text{out}}} \lambda_i.
\end{equation}
The proof and optimality result are provided in Appendix~\ref{sec:Proof_ESD_Truncation_Error}.

\textbf{Comparison.} Unlike direct SVD of the task update matrix (Equation~\ref{eq:svd_loss}), the ESD truncation error (Equation~\ref{eq:esd_loss}) is the sum of the discarded output-shift eigenvalues. Each $\lambda_i$ directly measures activation-shift energy along direction $e_i$, so truncation removes the directions carrying the least energy under $\mathcal{D}_{\text{proxy}}$. For any rank budget $r$, ESD minimizes the expected output error on this distribution over all rank-$r$ linear approximations; its error is therefore no greater than that of parameter-space SVD evaluated on the same inputs. When the proxy features are representative of the target distribution, this criterion is aligned with preserving target-task behavior. Experiments show that ESD yields higher energy concentration and feature retention than SVD (Fig.~\ref{fig:esd_vs_svd}) and is highly robust to the choice of proxy dataset (Figs.~\ref{fig:ablation_proxy_size} and \ref{fig:proxy_routing}; Table~\ref{tab:proxy_composition}).

\begin{figure*}[t]
    \centering
    \includegraphics[width=\linewidth]{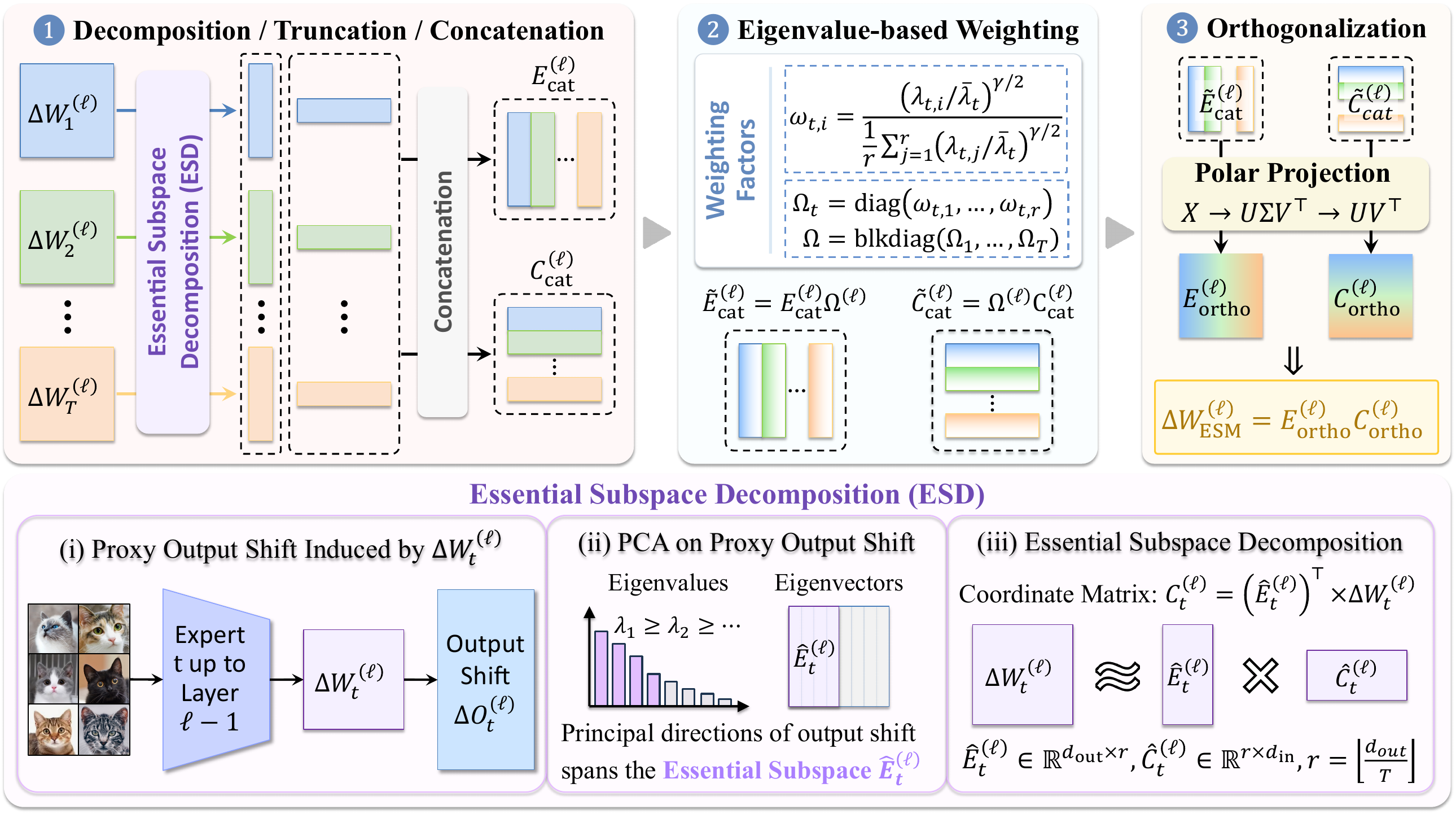}
    \caption{Overview of ESM, the proposed static training-free model merging method. Essential Subspace Decomposition (ESD) extracts and truncates output shift-aware factors for each task. The retained factors are concatenated and weighted according to their output-shift eigenvalues; Polar Projection then replaces each weighted factor with its orthogonal polar factor before the two are composed into the merged task matrix.}
    \label{fig:ESM_M}
\end{figure*}

\subsection{ESM: Static Essential Subspace Merging}
\label{sec:ESM}

Building upon ESD, ESM composes task matrices into a single merged model by retaining and aligning functionally important components, as illustrated in Fig.~\ref{fig:ESM_M}. This static essential merging process is training-free and follows the principle described above: remove non-essential directions before they accumulate into interference, preserve the principal directions that carry task knowledge, and orthogonalize the retained components to reduce conflict among tasks. The process follows three steps:

\begin{enumerate}
    \item \textbf{Decomposition and Truncation.} For each task $t \in \{1, \dots, T\}$, we factorize the task matrix $\Delta W_t$ into its essential basis $E_t$ and coordinate matrix $C_t$, such that $\Delta W_t = E_t C_t$, where $E_t \in \mathbb{R}^{d_{\text{out}} \times d_{\text{out}}}$ and $C_t \in \mathbb{R}^{d_{\text{out}} \times d_{\text{in}}}$. Given $T$ tasks, we allocate a rank budget $r = \lfloor d_{\text{out}} / T \rfloor$ to each one and retain the top-$r$ factors $\hat{E}_t \in \mathbb{R}^{d_{\text{out}} \times r}$ and $\hat{C}_t \in \mathbb{R}^{r \times d_{\text{in}}}$. We also quantify the relative importance of each retained direction using its ESD eigenvalue. Let $\lambda_{t,i}$ denote the eigenvalue of direction $i$ and $\bar{\lambda}_t=\frac{1}{r}\sum_{i=1}^{r}\lambda_{t,i}$. Its normalized importance weight is
    \begin{equation}\label{eq:esm_eigenvalue_weight}
    \omega_{t,i}=\frac{(\lambda_{t,i}/\bar{\lambda}_t)^{\gamma/2}}{\frac{1}{r}\sum_{j=1}^{r}(\lambda_{t,j}/\bar{\lambda}_t)^{\gamma/2}},
    \end{equation}
    where $\gamma$ is the eigenvalue-weighting power studied in Fig.~\ref{fig:eigen_power}. Since $\lambda_{t,i}$ measures the output-shift energy of the corresponding direction, $\omega_{t,i}$ encodes its preservation priority in the subsequent orthogonalization.

    \item \textbf{Concatenation.} We form the merged basis and coordinate matrices by horizontally and vertically concatenating the truncated factors across all tasks, respectively:
    \begin{equation}\label{eq:esm_concatenation}
    \begin{aligned}
    E_{\text{cat}} &= [\hat{E}_1 | \hat{E}_2 | \dots | \hat{E}_T]
    \in \mathbb{R}^{d_{\text{out}} \times (r \cdot T)}, \\
    C_{\text{cat}} &=
    \begin{bmatrix}
    \hat{C}_1 \\
    \hat{C}_2 \\
    \vdots \\
    \hat{C}_T
    \end{bmatrix}
    \in \mathbb{R}^{(r \cdot T) \times d_{\text{in}}}.
    \end{aligned}
    \end{equation}
    In the same task-wise order, we organize the importance weights into task-specific diagonal matrices and then assemble a global block-diagonal matrix:
    \begin{equation}\label{eq:esm_weight_matrix}
    \begin{aligned}
    \Omega_t&=\operatorname{diag}(\omega_{t,1},\ldots,\omega_{t,r}),\\
    \Omega&=\operatorname{blkdiag}(\Omega_1,\ldots,\Omega_T).
    \end{aligned}
    \end{equation}

    \item \textbf{Weighted Orthogonalization.} The concatenated factors $E_{\text{cat}}$ and $C_{\text{cat}}$ contain components from different task subspaces and may therefore be strongly correlated. To reduce this correlation while preferentially preserving functionally important directions, we use the assembled importance matrix $\Omega$ to weight both factors:
    \begin{equation}\label{eq:esm_weighted_factors}
    \begin{aligned}
    \tilde{E}_{\text{cat}}&=E_{\text{cat}}\Omega,\\
    \tilde{C}_{\text{cat}}&=\Omega C_{\text{cat}}.
    \end{aligned}
    \end{equation}
    Applying the same weight to the corresponding basis column and coordinate row scales direction $i$ by $\omega_{t,i}^2\propto(\lambda_{t,i}/\bar{\lambda}_t)^\gamma$ before projection. The orthogonal polar factors of these weighted matrices solve the corresponding weighted Orthogonal Procrustes problems: each direction's alignment score is scaled by $\omega_{t,i}$, so higher-weight directions exert greater influence on the projection. A formal derivation is provided in Appendix~\ref{sec:eigen_weight_theory}. We compute the SVDs of the weighted factors:
    \begin{equation}\label{eq:esm_orthogonalization_svd}
    \begin{aligned}
    \tilde{E}_{\text{cat}} &= U_E \Sigma_E V_E^{\top}, \\
    \tilde{C}_{\text{cat}} &= U_C \Sigma_C V_C^{\top}.
    \end{aligned}
    \end{equation}
    We then replace each weighted factor by its orthogonal polar factor $UV^{\top}$, which discards singular-value magnitudes while retaining the singular-vector alignment. Under the full-rank-in-the-smaller-dimension assumption stated in Appendix~\ref{sec:whitening_procrustes}, these factors are semi-orthogonal and solve the corresponding rectangular Orthogonal Procrustes problems:
    \begin{equation}\label{eq:esm_orthogonalization}
    \begin{aligned}
    E_{\text{ortho}} &= U_E V_E^{\top}, \\
    C_{\text{ortho}} &= U_C V_C^{\top}.
    \end{aligned}
    \end{equation}
    The final merged task matrix is constructed as follows:
    \begin{equation}\label{eq:esm_reconstruction}
    \Delta W_{\text{ESM}} = E_{\text{ortho}} C_{\text{ortho}}.
    \end{equation}
\end{enumerate}

The parameter matrix for the $\ell$-th layer of the final merged multi-task model is given by:
\begin{equation}\label{eq:esm_final_merge}
W_{\text{ESM}}^{(\ell)}=W_0^{(\ell)}+\alpha \cdot \Delta W_{\text{ESM}}^{(\ell)},
\end{equation}
where $\alpha$ is a global scaling coefficient selected on a held-out validation set as in previous model merging works \cite{marczakno,zhang2026dc}.

\begin{figure*}[t]
    \centering
    \includegraphics[width=\linewidth]{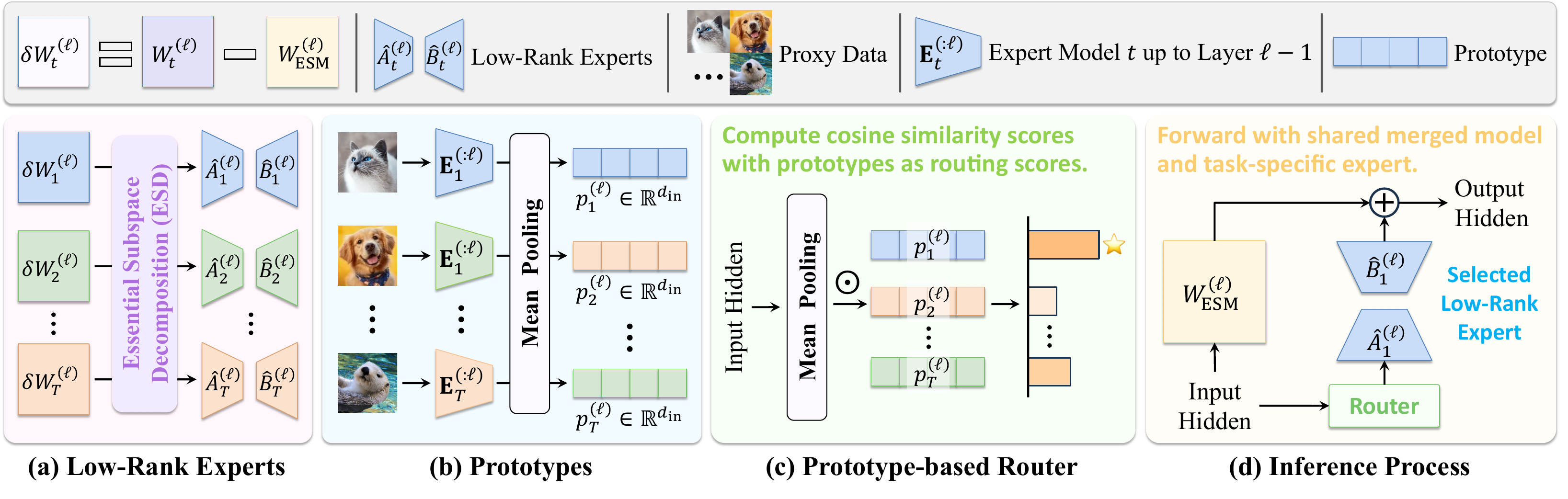}
    \caption{Overview of ESM++, the proposed dynamic training-free model merging method. ESM++ decomposes task-specific residuals into low-rank experts, collects per-task prototypes from proxy activations, and uses prototype-based routing to select the most relevant expert for each layer at inference time. The selected low-rank expert is composed with the shared ESM merged weight, preserving task-specific specialization while retaining shared knowledge.}
    \label{fig:ESM_R}
\end{figure*}

\subsection{ESM++: Dynamic Essential Subspace Merging}
\label{sec:ESM++}

ESM provides a shared base model for all target tasks. However, composing task updates into this model can dilute task-specific expertise because orthogonalization and inter-task competition may suppress task-unique knowledge. To address this issue, we introduce ESM++, which preserves low-rank residual components as separate experts and composes them dynamically at inference time, as shown in Fig.~\ref{fig:ESM_R}. ESM++ remains training-free: it learns no additional router and instead uses task prototypes collected offline. ESM and ESM++ share the same ESD principle but differ in how they compose the decomposed knowledge---statically in one model or dynamically through per-input, per-layer expert selection.

\subsubsection{\textbf{Expert Extraction}}

ESM++ builds on the shared base model $W_{\text{ESM}}$ produced by ESM. For each task $t$ and each weight matrix at layer $\ell$, we compute the residual task matrix, denoted by $\delta W_t$, as follows:
\begin{equation}\label{eq:esmr_residual}
\delta W_t = W_t - W_{\text{ESM}}.
\end{equation}
Unlike the original task matrices used in ESM, these residuals isolate task-specific information not captured by the static merged model. We decompose each residual $\delta W_t$ using ESD (Section~\ref{sec:ESD}), yielding a truncated essential basis $\hat{B}_t \in \mathbb{R}^{d_{\text{out}} \times r}$ and coordinate matrix $\hat{A}_t \in \mathbb{R}^{r \times d_{\text{in}}}$. Because the ESM base already captures much of the shared knowledge, a substantially smaller residual rank is sufficient empirically.

Each task retains its low-rank expert parameters $\{(\hat{B}_t, \hat{A}_t)\}$ across all target layers. At inference time, the ESM++ weight matrix for layer $\ell$ under expert $t$ is reconstructed as follows:
\begin{equation}\label{eq:esmr_expert_weight}
W_{\text{ESM++},t} = W_{\text{ESM}} + \hat{B}_t \hat{A}_t,
\end{equation}
where $W_{\text{ESM}}$ is the merged weight from ESM.

\subsubsection{\textbf{Prototype Collection}}

To route each input to a residual expert, we construct task prototypes offline from the same ESM-merged model used at inference. For each task $t$, we pass its unlabeled proxy samples through $W_{\text{ESM}}$ and capture the input features at each target layer. We mean-pool the tokens of each sample to obtain features $z_{t,i}^{(\ell)} \in \mathbb{R}^{d_{\text{in}}}$. Let $K_{\mathrm{p}}$ denote the number of prototypes per task and layer, reserving $k$ for the number of experts selected by top-$k$ routing. We then apply $k$-means to obtain $K_{\mathrm{p}}$ prototype centers:
\begin{equation}\label{eq:esmr_prototype}
\{p_{t,j}^{(\ell)}\}_{j=1}^{K_{\mathrm{p}}} = \operatorname{KMeans}_{K_{\mathrm{p}}}\left(\{z_{t,i}^{(\ell)}\}_{i=1}^{N}\right).
\end{equation}
Using the merged model aligns the prototype features with the shared representation used during deployment, while multiple centers capture within-task variation that a single mean may miss. The collection is performed once for all tasks and layers.

\subsubsection{\textbf{Per-Layer Routing at Inference}}

At inference time, we mean-pool the current layer features of each sample to obtain $\bar{x}^{(\ell)}$. We compare it with every prototype and select the task associated with the most similar center:
\begin{equation}\label{eq:esmr_routing_score}
s_{t,j}^{(\ell)} = \frac{(\bar{x}^{(\ell)})^{\top}p_{t,j}^{(\ell)}}{\|\bar{x}^{(\ell)}\|_2\,\|p_{t,j}^{(\ell)}\|_2},\qquad
(t^\star,j^\star) = \operatorname*{arg\,max}_{t,j} s_{t,j}^{(\ell)}.
\end{equation}
The winning task $t^\star$ selects the residual expert for that layer, and the residual branch is evaluated as $\hat{B}_{t^\star}(\hat{A}_{t^\star}x)$. Routing is performed independently at each target layer and requires no learned router or test-time optimization.

\section{Experiments}
\label{sec:experiments}

\begin{table*}[t]
\centering
\caption{Average accuracy on vision model merging benchmarks, with normalized accuracy shown as subscripts in parentheses. The Pre-trained row denotes the base model, and the Fine-tuned row denotes the task-specific expert reference.}
\label{tab:main_results}
\setlength{\tabcolsep}{4.5pt}
\resizebox{\textwidth}{!}{%
\begin{tabular}{lcccccccccc}
\toprule[1.5pt]
\multirow{2}{*}{Method} & \multirow{2}{*}{Venue} & \multicolumn{3}{c}{ViT-B/32} & \multicolumn{3}{c}{ViT-B/16} & \multicolumn{3}{c}{ViT-L/14}
\\
\cmidrule[0.5pt](lr){3-5}\cmidrule[0.5pt](lr){6-8}\cmidrule[0.5pt](lr){9-11} && 8 tasks & 14 tasks & 20 tasks & 8 tasks & 14 tasks & 20 tasks & 8 tasks & 14 tasks & 20 tasks
\\
\midrule[0.5pt]

\rowcolor[HTML]{ECECEC}\textit{Pre-trained} & -- & 48.3 & 57.2 & 56.1 & 55.3 & 61.3 & 59.7 & 64.7 & 68.2 & 65.2 \\
\rowcolor[HTML]{ECECEC}\textit{Fine-tuned} & -- & 92.8 & 90.9 & 91.3 & 94.6 & 92.8 & 93.2 & 95.8 & 94.3 & 94.7 \\

\midrule[0.5pt]

Model Soup \cite{wortsman2022model} & ICML 2022 & 66.3$_{(72.1)}$ & 64.3$_{(71.1)}$ & 61.0$_{(67.5)}$ & 72.2$_{(76.6)}$ & 69.5$_{(74.8)}$ & 65.3$_{(70.4)}$ & 79.6$_{(83.2)}$ & 76.7$_{(81.1)}$ & 71.6$_{(75.6)}$ \\
Task Arithmetic \cite{ilharcoediting} & ICLR 2023 & 70.8$_{(76.5)}$ & 65.3$_{(72.1)}$ & 60.5$_{(66.8)}$ & 75.4$_{(79.6)}$ & 70.5$_{(75.9)}$ & 65.8$_{(70.8)}$ & 84.9$_{(88.7)}$ & 79.4$_{(84.0)}$ & 74.0$_{(78.1)}$ \\
TIES-Merging \cite{yadav2023ties} & NeurIPS 2023 & 75.1$_{(81.0)}$ & 68.0$_{(74.8)}$ & 63.4$_{(69.9)}$ & 79.7$_{(84.3)}$ & 73.2$_{(78.7)}$ & 68.2$_{(73.3)}$ & 86.9$_{(90.7)}$ & 79.5$_{(84.1)}$ & 75.7$_{(79.8)}$ \\
Consensus TA \cite{wanglocalizing} & ICML 2024 & 75.0$_{(80.8)}$ & 70.4$_{(77.4)}$ & 65.4$_{(72.0)}$ & 79.4$_{(83.9)}$ & 74.4$_{(79.9)}$ & 69.8$_{(74.9)}$ & 86.3$_{(90.1)}$ & 82.2$_{(86.9)}$ & 79.0$_{(83.2)}$ \\
TSV-M \cite{gargiulo2025task} & CVPR 2025 & 85.9$_{(92.3)}$ & 80.1$_{(87.9)}$ & 77.1$_{(84.3)}$ & 89.0$_{(93.9)}$ & 84.6$_{(91.0)}$ & 80.6$_{(86.5)}$ & 93.0$_{(97.0)}$ & 89.2$_{(94.4)}$ & 87.7$_{(92.5)}$ \\
Iso-C \cite{marczakno} & ICML 2025 & 86.3$_{(92.9)}$ & 80.3$_{(88.1)}$ & 75.5$_{(82.5)}$ & 90.6$_{(95.6)}$ & 84.8$_{(91.1)}$ & 79.6$_{(85.4)}$ & 94.2$_{(98.3)}$ & 89.3$_{(94.5)}$ & 87.6$_{(92.2)}$ \\
Iso-CTS \cite{marczakno} & ICML 2025 & 86.2$_{(92.8)}$ & 81.7$_{(89.7)}$ & 78.1$_{(85.5)}$ & 91.1$_{(96.1)}$ & 86.4$_{(92.8)}$ & 82.4$_{(88.4)}$ & 94.7$_{(98.8)}$ & 91.0$_{(96.3)}$ & 90.1$_{(94.9)}$ \\
PACER \cite{chen2026rethinking} & TKDE 2026 & 72.3$_{(78.0)}$ & 66.0$_{(72.7)}$ & 64.9$_{(71.3)}$ & 76.9$_{(81.3)}$ & 71.2$_{(76.6)}$ & 69.9$_{(75.0)}$ & 86.2$_{(90.0)}$ & 79.9$_{(84.5)}$ & 76.5$_{(80.6)}$ \\
DC-Merge \cite{zhang2026dc} & CVPR 2026 & 87.1$_{(93.6)}$ & 82.5$_{(90.6)}$ & 80.6$_{(88.2)}$ & 90.8$_{(95.8)}$ & 87.1$_{(93.7)}$ & 84.6$_{(90.8)}$ & 94.3$_{(98.4)}$ & 91.0$_{(96.4)}$ & 90.5$_{(95.4)}$ \\

\rowcolor[HTML]{C6E2FF}ESM & Ours & 88.6$_{(95.4)}$ & 83.9$_{(92.4)}$ & 82.3$_{(90.1)}$ & 91.6$_{(96.7)}$ & 87.6$_{(94.4)}$ & 85.3$_{(91.6)}$ & 94.7$_{(98.8)}$ & 91.3$_{(96.8)}$ & 90.7$_{(95.7)}$ \\
\rowcolor[HTML]{C6E2FF}ESM++ ($r=8$) & Ours & 91.5$_{(98.5)}$ & 88.0$_{(96.8)}$ & 87.4$_{(95.6)}$ & 93.6$_{(98.8)}$ & 90.5$_{(97.5)}$ & 89.6$_{(96.1)}$ & 95.5$_{(99.7)}$ & 92.8$_{(98.4)}$ & 92.9$_{(98.0)}$ \\
\rowcolor[HTML]{C6E2FF}ESM++ ($r=32$) & Ours & \textbf{92.1}$_{(\textbf{99.2})}$ & \textbf{89.0}$_{(\textbf{97.8})}$ & \textbf{88.6}$_{(\textbf{96.8})}$ & \textbf{93.9}$_{(\textbf{99.2})}$ & \textbf{91.2}$_{(\textbf{98.2})}$ & \textbf{90.8}$_{(\textbf{97.4})}$ & \textbf{95.7}$_{(\textbf{99.9})}$ & \textbf{93.3}$_{(\textbf{98.9})}$ & \textbf{93.5}$_{(\textbf{98.6})}$ \\

\bottomrule[1.5pt]
\end{tabular}%
}
\end{table*}

\subsection{Experimental Setup}
\textbf{Vision model merging.} Following \cite{wanglocalizing}, we evaluate multi-task merging on benchmarks of 8, 14, and 20 vision tasks. The 8-task benchmark includes Cars \cite{krause20133d}, DTD \cite{cimpoi2014describing}, EuroSAT \cite{helber2019eurosat}, GTSRB \cite{stallkamp2011german}, MNIST \cite{lecun2002gradient}, RESISC45 \cite{cheng2017remote}, SUN397 \cite{xiao2016sun}, and SVHN \cite{netzer2011reading}. The 14-task benchmark further adds CIFAR100 \cite{krizhevsky2009learning}, STL10 \cite{coates2011analysis}, Flowers102 \cite{nilsback2008automated}, OxfordIIITPet \cite{parkhi2012cats}, PCAM \cite{veeling2018rotation}, and FER2013 \cite{goodfellow2013challenges}. The 20-task benchmark additionally includes EMNIST \cite{cohen2017emnist}, CIFAR10 \cite{krizhevsky2009learning}, Food101 \cite{bossard2014food}, FashionMNIST \cite{xiao2017fashion}, RenderedSST2 \cite{socher2013recursive}, and KMNIST \cite{clanuwat2018deep}. We use CLIP \cite{radford2021learning} models with ViT-B/32, ViT-B/16, and ViT-L/14 visual encoders as pre-trained base models, and adopt the task-specific fine-tuned checkpoints provided by the TALL-masks \cite{wanglocalizing}. We report both absolute and normalized accuracy following standard evaluation practices \cite{wanglocalizing}.

\textbf{Discriminative language model merging.} For discriminative language tasks, we evaluate an eight-task subset of GLUE \cite{wang2018glue}, excluding WNLI, using RoBERTa-Base \cite{liu2019roberta} as the pre-trained base model. We report Matthews correlation for CoLA, accuracy for SST-2, MNLI, QNLI, and RTE, the mean of accuracy and F1 for MRPC and QQP, and the mean of Pearson and Spearman correlation for STS-B. The Avg. column is the unweighted mean of the eight task scores.

\textbf{Generative language model merging.} For generative language tasks, we follow MergeBench \cite{he2025mergebench} and evaluate on instruction-following, mathematics, multilingual understanding, coding, and safety tasks as listed in Appendix~\ref{sec:generative_lm_eval_datasets}. We use Llama-3.2-3B \cite{grattafiori2024llama} as the pre-trained base model.

\textbf{Implementation details.} Throughout this paper, ``training-free'' means that no model parameters or routing networks are optimized by gradient descent. Following prior model-merging methods \cite{marczakno,zhang2026dc}, we use the validation set to select the global scaling coefficient $\alpha$; the selected values are reported in the supplementary material. Proxy data consist of unlabeled examples randomly sampled from this validation set. Unless otherwise stated, we use 32 proxy examples per task. We set the eigenvalue-weighting power to $\gamma=0.3$. ESM uses the layer-wise rank budget $r=\lfloor d_{\text{out}}/T\rfloor$, while the residual rank used by ESM++ is reported with each result. For ESM++, prototypes are extracted from the ESM-merged model by token mean pooling followed by $k$-means; we use $K_{\mathrm{p}}=1$ for vision and $K_{\mathrm{p}}=8$ for language tasks. Normalized performance is computed for each task relative to its corresponding fine-tuned expert and then averaged across tasks.

\subsection{Main Results}

\begin{table*}[t]
\centering
\caption{Multi-task scores when merging RoBERTa-Base models on an eight-task subset of GLUE, excluding WNLI.}
\label{tab:roberta}
\setlength{\tabcolsep}{5.2pt}
\resizebox{\textwidth}{!}{%
\begin{tabular}{lccccccccc}
\toprule[1.5pt]
Method & CoLA & SST-2 & MRPC & STS-B & QQP & MNLI & QNLI & RTE & Avg. \\
\midrule

\rowcolor[HTML]{ECECEC}\textit{Pre-trained} & 0.0 & 49.1 & 15.8 & 15.0 & 41.1 & 34.2 & 52.4 & 53.4 & 32.6 \\
\rowcolor[HTML]{ECECEC}\textit{Fine-tuned} & 56.5 & 94.7 & 88.0 & 86.4 & 89.7 & 87.0 & 91.7 & 66.4 & 82.6 \\

\midrule

Model Soup \cite{wortsman2022model} & 0.0$_{(0.0)}$ & 56.1$_{(59.2)}$ & 75.5$_{(85.8)}$ & 40.6$_{(47.0)}$ & 40.7$_{(45.4)}$ & 41.8$_{(48.0)}$ & 58.6$_{(63.9)}$ & 47.3$_{(71.2)}$ & 45.1$_{(52.6)}$ \\
Task Arithmetic \cite{ilharcoediting} & 6.7$_{(11.8)}$ & 83.9$_{(88.6)}$ & 78.4$_{(89.1)}$ & 27.9$_{(32.3)}$ & 73.2$_{(81.6)}$ & 65.8$_{(75.7)}$ & 78.4$_{(85.5)}$ & 53.4$_{(80.4)}$ & 58.5$_{(68.1)}$ \\
TIES-Merging \cite{yadav2023ties} & 17.8$_{(31.8)}$ & 84.2$_{(88.9)}$ & 75.9$_{(86.2)}$ & 9.4$_{(10.9)}$ & 54.8$_{(61.1)}$ & 72.2$_{(83.0)}$ & 78.8$_{(85.9)}$ & 46.2$_{(69.6)}$ & 54.9$_{(64.7)}$ \\
DARE \cite{yu2024language} (w/ Task Arithmetic) & 0.0$_{(0.0)}$ & 83.4$_{(88.1)}$ & 76.2$_{(86.6)}$ & 26.1$_{(30.2)}$ & 75.6$_{(84.3)}$ & 55.7$_{(64.0)}$ & 72.5$_{(79.1)}$ & 51.3$_{(77.2)}$ & 55.1$_{(63.7)}$ \\
DARE \cite{yu2024language} (w/ TIES-Merging) & 6.7$_{(11.8)}$ & 90.4$_{(95.5)}$ & 75.5$_{(85.8)}$ & 8.1$_{(9.4)}$ & 77.9$_{(86.8)}$ & 72.3$_{(83.1)}$ & 81.3$_{(88.7)}$ & 43.6$_{(65.6)}$ & 57.0$_{(65.6)}$ \\
CAT-Merging \cite{sun2025cat} & 33.2$_{(58.8)}$ & 89.3$_{(94.3)}$ & 68.2$_{(77.5)}$ & 15.6$_{(18.1)}$ & 76.1$_{(84.8)}$ & 72.3$_{(83.1)}$ & 82.9$_{(90.4)}$ & 62.8$_{(94.6)}$ & 62.6$_{(75.2)}$ \\
LOT-Merging \cite{sun2025towards} & 17.1$_{(30.3)}$ & 89.7$_{(94.7)}$ & 78.3$_{(89.0)}$ & 25.5$_{(29.5)}$ & 78.8$_{(87.8)}$ & 73.0$_{(83.9)}$ & 77.2$_{(84.2)}$ & \underline{65.4}$_{(\underline{98.5})}$ & 63.1$_{(74.7)}$ \\
TSV-M \cite{gargiulo2025task} & 25.4$_{(44.8)}$ & \underline{93.4}$_{(\underline{98.6})}$ & 80.4$_{(91.4)}$ & 43.8$_{(50.7)}$ & 83.1$_{(92.7)}$ & \underline{74.6}$_{(\underline{85.8})}$ & 84.6$_{(92.3)}$ & 50.5$_{(76.1)}$ & 67.0$_{(79.0)}$ \\
WUDI-Merging \cite{chengwhoever} & \underline{46.2}$_{(\underline{81.8})}$ & 93.1$_{(98.3)}$ & 69.3$_{(78.7)}$ & 52.3$_{(60.5)}$ & \underline{83.2}$_{(\underline{92.7})}$ & \textbf{81.2}$_{(\textbf{93.3})}$ & 83.0$_{(90.5)}$ & 57.4$_{(86.4)}$ & 70.7$_{(85.3)}$ \\

\rowcolor[HTML]{C6E2FF}ESM (Ours) & 42.2$_{(74.7)}$ & 91.1$_{(96.2)}$ & \underline{83.9}$_{(\underline{95.3})}$ & \underline{74.0}$_{(\underline{85.6})}$ & 76.6$_{(85.4)}$ & 73.3$_{(84.3)}$ & \underline{85.0}$_{(\underline{92.7})}$ & \textbf{68.2}$_{(\textbf{102.7})}$ & \underline{74.3}$_{(\underline{89.6})}$ \\

\rowcolor[HTML]{C6E2FF}ESM++ ($r=32$, Ours) & \textbf{59.2}$_{(\textbf{104.7})}$ & \textbf{94.2}$_{(\textbf{99.4})}$ & \textbf{88.3}$_{(\textbf{100.4})}$ & \textbf{78.8}$_{(\textbf{91.3})}$ & \textbf{86.5}$_{(\textbf{96.4})}$ & 77.7$_{(89.3)}$ & \textbf{88.7}$_{(\textbf{96.7})}$ & 63.2$_{(95.1)}$ & \textbf{79.6}$_{(\textbf{96.7})}$ \\

\bottomrule[1.5pt]
\end{tabular}
}
\end{table*}

\begin{table*}[t]
\centering
\caption{Results for merging five Llama-3.2-3B experts across instruction following, mathematics, multilingual understanding, coding, and safety. For multi-dataset categories, entries show the category average followed by per-dataset scores in header order.}
\label{tab:llm_merging}
\scriptsize
\setlength{\tabcolsep}{1pt}
\resizebox{\textwidth}{!}{%
\begin{tabular}{l>{\centering\arraybackslash}m{1.3cm}>{\centering\arraybackslash}m{1.5cm}>{\centering\arraybackslash}m{3.1cm}>{\centering\arraybackslash}m{2.5cm}>{\centering\arraybackslash}m{3.8cm}>{\centering\arraybackslash}m{0.9cm}}
\toprule[1.5pt]
\multirow{2}{*}{Method} & Instruction & Mathematics & Multilingual & Coding & Safety & \multirow{2}{*}{Avg.} \\
\cmidrule(lr){2-2}\cmidrule(lr){3-3}\cmidrule(lr){4-4}\cmidrule(lr){5-5}\cmidrule(lr){6-6}
& \makecell[c]{\rotatebox{55}{IFEval}} & \makecell[c]{\rotatebox{55}{GSM8K}} & \makecell[c]{\rotatebox{55}{M\_MMLU}\hspace{1pt}\rotatebox{55}{M\_ARC}\hspace{1pt}\rotatebox{55}{M\_Hellaswag}} & \makecell[c]{\rotatebox{55}{Humaneval+}\hspace{2pt}\rotatebox{55}{MBPP+}} & \makecell[c]{\rotatebox{55}{WildGuardTest}\hspace{1pt}\rotatebox{55}{HarmBench}\hspace{1pt}\rotatebox{55}{DoAnythingNow}\hspace{1pt}\rotatebox{55}{XSTest}} & \\
\midrule[0.5pt]

\rowcolor[HTML]{ECECEC}\textit{Pre-trained} & 10.45 & 30.25 & 40.73 (46.58; 33.23; 42.38) & 27.17 (15.06; 39.29) & 19.87 (30.04; 16.56; 16.00; 16.89) & 25.69 \\
\rowcolor[HTML]{ECECEC}\textit{Fine-tuned} & 53.52 & 60.27 & 41.64 (46.22; 35.66; 43.04) & 44.62 (42.32; 46.93) & 42.23 (54.34; 14.69; 80.33; 19.56) & 48.46 \\

\midrule[0.5pt]

Model Soup \cite{wortsman2022model} & 13.71 & 41.93 & 42.25 (47.27; 35.54; 43.95) & 37.22 (31.16; 43.28) & 31.21 (35.38; 17.81; 37.67; 34.00) & 33.26 \\
Task Arithmetic \cite{ilharcoediting} & 35.91 & 51.02 & \underline{42.35} (46.54; 36.01; 44.51) & 41.17 (36.34; 46.01) & 40.16 (41.66; 35.31; 29.00; 54.67) & 42.12 \\
TIES-Merging \cite{yadav2023ties} & 20.66 & 48.07 & 42.34 (46.98; 35.96; 44.09) & 39.05 (33.60; 44.50) & 33.73 (36.45; 24.69; 33.33; 40.44) & 36.77 \\
Consensus TA \cite{wanglocalizing} & 33.12 & 48.07 & \textbf{42.39} (46.83; 35.81; 44.52) & \textbf{42.05} (37.99; 46.11) & 36.87 (38.85; 29.06; 27.33; 52.22) & 40.50 \\
L\&S \cite{he2024localize} & 29.51 & 45.26 & 42.15 (47.17; 35.88; 43.39) & 33.71 (25.85; 41.56) & 25.33 (25.50; 11.25; 30.33; 34.22) & 35.19 \\
DARE \cite{yu2024language} & 33.53 & 50.11 & 42.13 (46.34; 35.62; 44.42) & 39.94 (35.24; 44.63) & 38.96 (40.59; 33.13; 29.67; 52.44) & 40.93 \\
TSV-M \cite{gargiulo2025task} & 41.29 & 51.78 & 41.28 (45.42; 34.70; 43.72) & 39.24 (33.72; 44.76) & 45.24 (47.80; 40.94; 34.00; 58.22) & 43.77 \\
\rowcolor[HTML]{C6E2FF}ESM (Ours) & \underline{43.30} & \underline{54.13} & 41.33 (45.00; 35.19; 43.81) & 40.19 (33.48; 46.90) & \textbf{47.96} (49.80; 42.50; 37.33; 62.22) & \underline{45.38} \\
\rowcolor[HTML]{C6E2FF}ESM++ ($r=64$, Ours) & \textbf{47.56} & \textbf{60.73} & 41.92 (46.15; 36.18; 43.43) & \underline{42.00} (37.99; 46.01) & \underline{45.87} (50.07; 28.75; 58.67; 46.00) & \textbf{47.61} \\

\bottomrule[1.5pt]
\end{tabular}%
}
\end{table*}






\textbf{Vision model merging.}
As presented in Table~\ref{tab:main_results}, we compare the proposed methods against representative model merging baselines, using the pre-trained base model and the average single-task fine-tuned performance as reference points.
ESM matches or outperforms all comparison methods across the nine settings. By preserving task-specific residual expertise, ESM++ further improves upon ESM, with the $r=32$ variant obtaining the best results among all merging methods.
For vision, we use one prototype per task and layer because increasing the prototype count has negligible effect; prototype-count ablations are provided in Appendix~\ref{sec:prototype_ablation}. In the eight-task leave-one-task-out OOD evaluation, increasing $K_{\mathrm{p}}$ from 1 to 2 improves average absolute and normalized accuracy by 0.84 and 0.90 points, respectively, with limited gains beyond $K_{\mathrm{p}}=2$; the detailed results are provided in Appendix~\ref{sec:ood_results}. The advantage of ESM++ grows with the number of tasks, highlighting its ability to mitigate inter-task interference.

\textbf{Discriminative language model merging.}
Table~\ref{tab:roberta} reports results on the eight-task GLUE subset \cite{wang2018glue} with RoBERTa-Base \cite{liu2019roberta}.
ESM already surpasses prior static merging methods, achieving an average GLUE score of $74.3$.
ESM++ raises the average score to $79.6$ (96.7 normalized) and achieves the best performance on six of the eight datasets.

\textbf{Generative large language model merging.}
Following MergeBench \cite{he2025mergebench}, Table~\ref{tab:llm_merging} evaluates merging five Llama-3.2-3B expert models.
ESM achieves the highest average score among static methods.
ESM++ further improves the average score to $47.61$, approaching the fine-tuned expert reference of $48.46$.
The consistent gains across these diverse generative capabilities demonstrate that ESM can merge specialized LLM experts while preserving complementary task knowledge.

\begin{figure}[t]
    \centering
    \includegraphics[width=\linewidth]{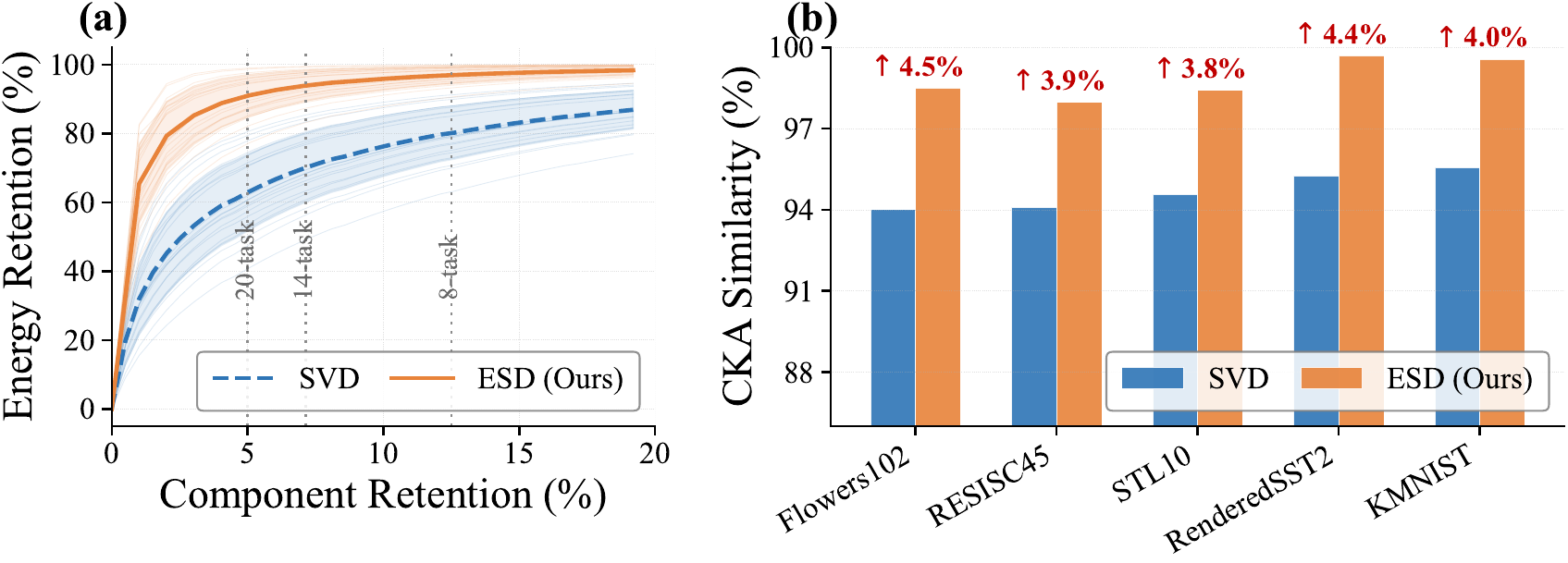}
    \caption{ESD and parameter-space SVD on ViT-B/16. (a) Cumulative energy retention versus retained-component fraction; ESD retains more energy with fewer components. (b) CKA similarity to the fine-tuned expert after truncation; higher values indicate better retention of task-specific features.}
    \label{fig:esd_vs_svd}
\end{figure}

\subsection{Ablation and Analysis}

\subsubsection{\textbf{Comparison of ESD and Parameter-Space SVD}}

Fig.~\ref{fig:esd_vs_svd} compares ESD with parameter-space SVD on ViT-B/16. \figpanel{fig:esd_vs_svd}{a} plots cumulative energy retention (defined in Appendix~\ref{sec:energy_retention}) against the proportion of retained components. ESD concentrates energy more strongly, retaining more energy with fewer components. \figpanel{fig:esd_vs_svd}{b} compares feature preservation after low-rank decomposition using Centered Kernel Alignment (CKA) similarity \cite{kornblith2019similarity}. We measure the similarity between the low-rank decomposed model and the fine-tuned expert using the class token from the final layer, based on the feature difference relative to the zero-shot pre-trained model. ESD achieves higher CKA similarity than parameter-space SVD, indicating better feature retention after truncation.

\begin{table*}[t]
\centering
\caption{Ablation study of key components in the proposed ESM merging method.}
\label{tab:ablation}
\setlength{\tabcolsep}{3.6pt}
\resizebox{\textwidth}{!}{%
\begin{tabular}{>{\centering\arraybackslash}p{1cm}>{\centering\arraybackslash}p{1cm}cc|ccccccccccc}
\toprule[1.5pt]
\multicolumn{2}{c}{Decomposition} & \multirow{2}{*}{Truncation} & \multirow{2}{*}{Orthogonalization} & \multicolumn{3}{c}{ViT-B/16} & \multicolumn{6}{c}{Llama-3.2-3B} & \multirow{2}{*}{RoBERTa}
\\
\cmidrule[0.5pt](lr){1-2}\cmidrule[0.5pt](lr){5-7}\cmidrule[0.5pt](lr){8-13} SVD & ESD &&& 8 tasks & 14 tasks & 20 tasks & Instruction & Math & Coding & Multilingual & Safety & Avg.
\\
\midrule[0.5pt]

\textcolor{red}{\ding{55}} & \textcolor{green}{\ding{51}} & \textcolor{red}{\ding{55}} & \textcolor{red}{\ding{55}} & 80.3 & 74.3 & 72.9 & 31.9 & 48.1 & 41.0 & \textbf{42.0} & 39.9 & 40.6 & 58.1 \\

\textcolor{red}{\ding{55}} & \textcolor{green}{\ding{51}} & \textcolor{green}{\ding{51}} & \textcolor{red}{\ding{55}} & 79.9 & 73.9 & 72.6 & 32.9 & 47.7 & \textbf{41.3} & 41.9 & 39.4 & 40.6 & 52.6 \\

\textcolor{red}{\ding{55}} & \textcolor{green}{\ding{51}} & \textcolor{red}{\ding{55}} & \textcolor{green}{\ding{51}} & 89.1 & 83.5 & 79.4 & 33.7 & 52.1 & 41.1 & \textbf{42.0} & 40.4 & 41.8 & 65.1 \\

\textcolor{green}{\ding{51}} & \textcolor{red}{\ding{55}} & \textcolor{green}{\ding{51}} & \textcolor{green}{\ding{51}} & 89.6 & 85.4 & 82.1 & 38.3 & 51.8 & 39.2 & 41.3 & 45.2 & 43.2 & 66.3 \\

\rowcolor[HTML]{C6E2FF}\textcolor{red}{\ding{55}} & \textcolor{green}{\ding{51}} & \textcolor{green}{\ding{51}} & \textcolor{green}{\ding{51}} & \textbf{91.6} & \textbf{87.6} & \textbf{85.3} & \textbf{43.3} & \textbf{54.1} & 40.2 & 41.3 & \textbf{48.0} & \textbf{45.4} & \textbf{74.3} \\

\bottomrule[1.5pt]
\end{tabular}
}
\end{table*}

\begin{figure*}[t]
    \centering
    \includegraphics[width=\linewidth]{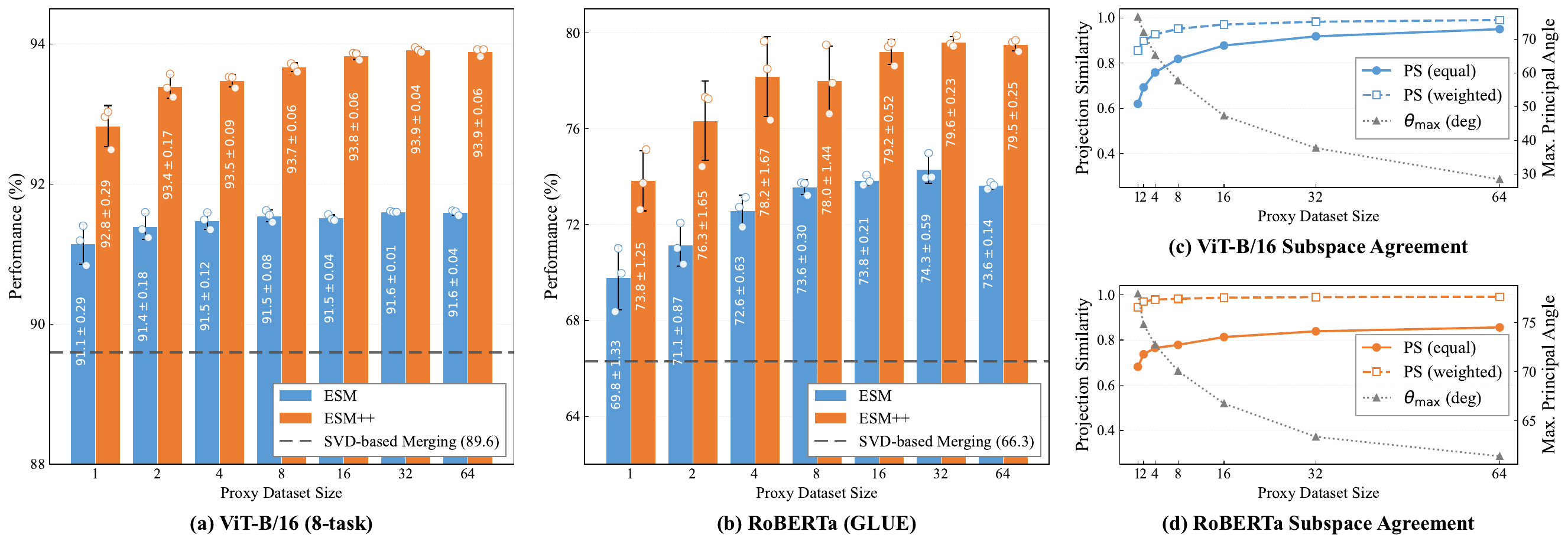}
    \caption{Impact of proxy dataset size on merging performance and subspace estimation. Panels (a,b) report average performance on ViT-B/16 and RoBERTa, while panels (c,d) compare proxy-estimated subspaces with full-test-set reference subspaces using projection similarity ``PS'' and maximum principal angle ``$\theta_{\max}$''. Among them, the weighted projection similarity ``PS (weighted)'' is the most relevant indicator because it measures how much high-energy functional information is preserved. Detailed definitions of these subspace metrics are provided in the Appendix~\ref{sec:proxy_subspace_metrics}.}
    \label{fig:ablation_proxy_size}
\end{figure*}

\subsubsection{\textbf{Ablation of ESM Components}}
\label{sec:ablation}

We conduct an ablation study of the three key components in ESM---the decomposition method, truncation, and orthogonalization---as shown in Table~\ref{tab:ablation}. First, replacing direct SVD of parameter updates with the proposed output-shift-aware ESD improves most reported scores. Under the same truncation and orthogonalization setting, ESD improves the reported scores over SVD except for the multilingual score, which is unchanged; this supports retaining functionally important rather than parameter-dominant directions. Second, orthogonalization substantially improves the merged model by reducing interference among task-specific components. Fig.~\ref{fig:eigen_power} further supports the effectiveness of weighting parameter orthogonalization by output-shift energy.
Third, truncation alone can reduce performance, but it provides clear gains when combined with orthogonalization.
This indicates that low-rank truncation is most beneficial as a preparation for orthogonalization: it filters out weak and interference-prone directions, allowing the orthogonalized factors to preserve the main task knowledge more effectively.

\subsubsection{\textbf{Impact of Proxy Dataset Size}}

We study the effect of proxy-set size in Fig.~\ref{fig:ablation_proxy_size}. Panels (a) and (b) show ESM and ESM++ performance as the number of proxy samples varies. A single unlabeled sample already outperforms data-free SVD of the parameter updates, with performance stabilizing after only a few samples. Panels (c) and (d) compare proxy-estimated subspaces with references from the full test set. The eigenvalue-weighted projection similarity ``PS (weighted)'' remains high with few samples, indicating reliable recovery of dominant directions. By contrast, the maximum principal angle ``$\theta_{\max}$'' is more sensitive to low-energy tail directions, whose worst-case alignment is harder to estimate despite their smaller contribution to retained information.

\begin{figure}[t]
    \centering
    \includegraphics[width=\linewidth]{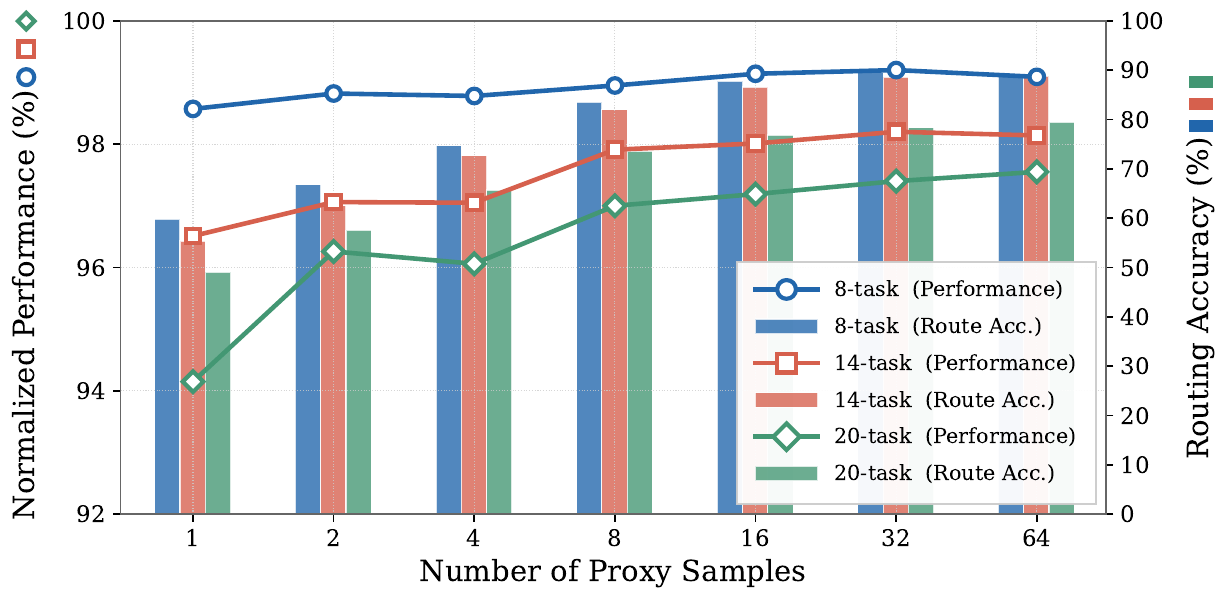}
    \caption{Effect of proxy set size on ESM++. We report the normalized accuracy and routing accuracy obtained when varying the number of proxy samples.}
    \label{fig:proxy_routing}
\end{figure}

Fig.~\ref{fig:proxy_routing} studies the sensitivity of ESM++ to the number of proxy samples used for prototype and residual expert construction. With only one unlabeled proxy sample, ESM++ already preserves more than $90\%$ of the expert-model performance. With a small proxy set, both routing accuracy and performance become stable, suggesting that ESM++ does not require a large proxy dataset for effective routing and expert composition.

\begin{figure*}[t]
    \centering
    \includegraphics[width=\linewidth]{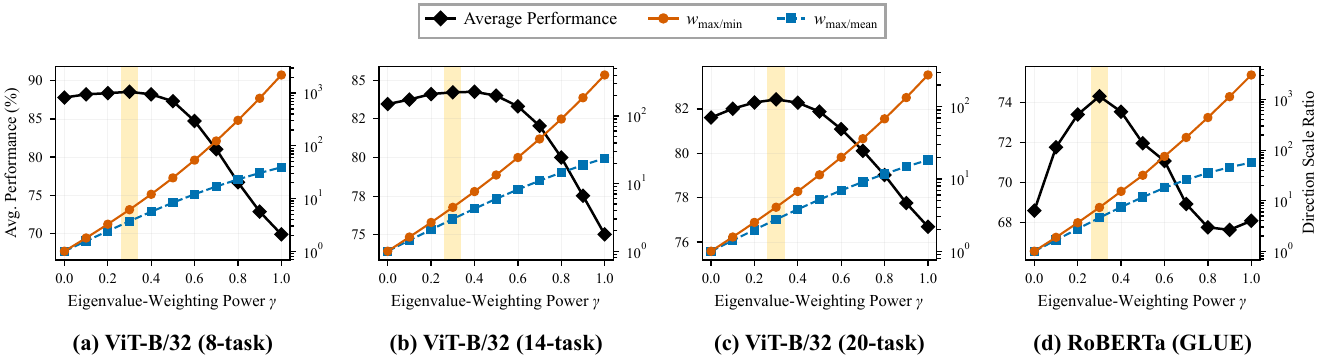}
    \caption{Effect of the eigenvalue-weighting power $\gamma$ before orthogonalization. Here, $w_{\mathrm{max/min}}$ denotes the ratio between the weighting coefficient of the largest direction and that of the smallest direction, while $w_{\mathrm{max/mean}}$ denotes the ratio between the weighting coefficient of the largest direction and the average weighting coefficient over all directions. The \colorbox[HTML]{FFF0C1}{gold} marker highlights the best setting, $\gamma=0.3$.}
    \label{fig:eigen_power}
\end{figure*}

\begin{table*}[t]
\centering
\caption{ESM++ results on the eight-task GLUE subset with RoBERTa, comparing residual ranks, shared bases, and routing strategies.
}
\label{tab:base_model}
\setlength{\tabcolsep}{1.7pt}
\resizebox{\textwidth}{!}{%
\begin{tabular}{lcccccccccccc}
\toprule[1.5pt]
Method & \makecell{Routing\\Strategy} & \makecell{Base\\Model} & \makecell{Routing\\Accuracy} & CoLA & SST-2 & MRPC & STS-B & QQP & MNLI & QNLI & RTE & Avg. \\
\midrule

\rowcolor[HTML]{ECECEC}\textit{Pre-trained} & -- & -- & -- & 0.0 & 49.1 & 15.8 & 15.0 & 41.1 & 34.2 & 52.4 & 53.4 & 32.6 \\
\rowcolor[HTML]{ECECEC}\textit{Fine-tuned} & -- & -- & -- & 56.5 & 94.7 & 88.0 & 86.4 & 89.7 & 87.0 & 91.7 & 66.4 & 82.6 \\
\rowcolor[HTML]{ECECEC}ESM & -- & -- & -- & 42.2$_{(74.7)}$ & 91.1$_{(96.2)}$ & 83.9$_{(95.3)}$ & 74.0$_{(85.6)}$ & 76.6$_{(85.4)}$ & 73.3$_{(84.3)}$ & 85.0$_{(92.7)}$ & 68.2$_{(102.7)}$ & 74.3$_{(89.6)}$ \\

\midrule

\multirow{4}{*}{ESM++ ($r=8$)}
& \multirow{2}{*}{Prototype} & Pre-trained & 71.2\% & 56.0$_{(99.1)}$ & 89.3$_{(94.3)}$ & 82.4$_{(93.6)}$ & 71.1$_{(82.3)}$ & 80.9$_{(90.2)}$ & 47.4$_{(54.5)}$ & 75.8$_{(82.7)}$ & 57.0$_{(85.8)}$ & 70.0$_{(84.7)}$ \\
& & ESM-merged & \textbf{76.2\%} & \textbf{56.5}$_{(\textbf{100.0})}$ & \textbf{94.2}$_{(\textbf{99.4})}$ & \textbf{88.5}$_{(\textbf{100.6})}$ & \textbf{78.5}$_{(\textbf{90.9})}$ & \textbf{85.9}$_{(\textbf{95.7})}$ & \textbf{78.1}$_{(\textbf{89.8})}$ & \textbf{89.7}$_{(\textbf{97.8})}$ & \textbf{63.5}$_{(\textbf{95.7})}$ & \textbf{79.4}$_{(\textbf{96.2})}$ \\

\cmidrule(l){2-13}
& \multirow{2}{*}{Oracle} & Pre-trained & 100\% & 56.0$_{(99.1)}$ & 93.3$_{(98.5)}$ & 86.7$_{(98.5)}$ & 85.3$_{(98.7)}$ & 82.3$_{(91.8)}$ & 83.3$_{(95.7)}$ & 89.2$_{(97.3)}$ & \textbf{65.7}$_{(\textbf{98.9})}$ & 80.2$_{(97.1)}$ \\
& & ESM-merged & 100\% & \textbf{56.2}$_{(\textbf{99.4})}$ & \textbf{94.5}$_{(\textbf{99.8})}$ & \textbf{87.9}$_{(\textbf{99.9})}$ & \textbf{86.3}$_{(\textbf{99.9})}$ & \textbf{87.5}$_{(\textbf{97.5})}$ & \textbf{86.8}$_{(\textbf{99.8})}$ & \textbf{90.5}$_{(98.7)}$ & 65.3$_{(98.4)}$ & \textbf{81.9}$_{(\textbf{99.2})}$ \\

\midrule

\multirow{4}{*}{ESM++ ($r=32$)}
& \multirow{2}{*}{Prototype} & Pre-trained & 71.1\% & \textbf{56.5}$_{(\textbf{100.0})}$ & 89.9$_{(94.9)}$ & 82.3$_{(93.5)}$ & 71.5$_{(82.8)}$ & 84.7$_{(94.4)}$ & 49.3$_{(56.7)}$ & 77.2$_{(84.2)}$ & \textbf{58.5}$_{(\textbf{88.1})}$ & 71.2$_{(86.2)}$ \\
& & ESM-merged & \textbf{75.8\%} & \textbf{59.2}$_{(\textbf{104.7})}$ & \textbf{94.2}$_{(\textbf{99.4})}$ & \textbf{88.3}$_{(\textbf{100.4})}$ & \textbf{78.8}$_{(\textbf{91.3})}$ & \textbf{86.5}$_{(\textbf{96.4})}$ & \textbf{77.7}$_{(\textbf{89.3})}$ & \textbf{88.7}$_{(\textbf{96.7})}$ & \textbf{63.2}$_{(\textbf{95.1})}$ & \textbf{79.6}$_{(\textbf{96.7})}$ \\

\cmidrule(l){2-13}
& \multirow{2}{*}{Oracle} & Pre-trained & 100\% & 56.3$_{(99.6)}$ & \textbf{94.3}$_{(\textbf{99.6})}$ & 87.3$_{(99.2)}$ & 86.1$_{(99.7)}$ & 86.5$_{(96.4)}$ & 85.8$_{(98.6)}$ & 90.6$_{(98.8)}$ & 66.8$_{(100.6)}$ & 81.7$_{(98.9)}$ \\
& & ESM-merged & 100\% & \textbf{58.0}$_{(\textbf{102.7})}$ & \textbf{94.3}$_{(99.5)}$ & \textbf{88.6}$_{(\textbf{100.7})}$ & \textbf{86.5}$_{(\textbf{100.1})}$ & \textbf{88.3}$_{(\textbf{98.4})}$ & \textbf{85.9}$_{(\textbf{98.7})}$ & \textbf{91.0}$_{(\textbf{99.2})}$ & \textbf{67.5}$_{(\textbf{101.6})}$ & \textbf{82.5}$_{(\textbf{100.1})}$ \\

\bottomrule[1.5pt]
\end{tabular}
}
\end{table*}

\begin{table}[t]
\centering
\caption{Impact of proxy dataset composition on merging performance.
}
\label{tab:proxy_composition}
\setlength{\tabcolsep}{3.6pt}
\resizebox{\linewidth}{!}{%
\begin{tabular}{cc|cccc}
\toprule[1.5pt]
\makecell{Decomp.\\Method} & \makecell{Sampling\\Strategy} & ViT-B/32 & ViT-B/16 & ViT-L/14 & RoBERTa
\\
\midrule[0.5pt]


\rowcolor[HTML]{ECECEC}SVD & - & 86.6 & 89.6 & 93.4 & 67.0\\
\rowcolor[HTML]{C6E2FF}ESD & \textit{Random (ID)} & 88.6 & 91.6 & 94.7 & 74.3 \\
ESD & \textit{Class Imbalance} & 88.6 & 91.6 & 94.7 & 72.9 \\
ESD & \textit{Random (OOD)} & 88.5 & 91.6 & 94.7 & 74.2 \\



\bottomrule[1.5pt]
\end{tabular}
}
\end{table}

\subsubsection{\textbf{Impact of Proxy Dataset Composition}}

We analyze the effect of proxy dataset composition on ESM using 32 unlabeled proxy samples per task. Table~\ref{tab:proxy_composition} reports performance on the 8-task ViT benchmarks and the eight-task GLUE subset with RoBERTa.
Our default setup uses unlabeled samples randomly selected from the corresponding task dataset, denoted as ``\textit{Random (ID)}''. We also consider two challenging scenarios: sampling only from a single class within each task dataset (``\textit{Class Imbalance}'') and sampling from an out-of-distribution dataset (ImageNet-1k \cite{russakovsky2015imagenet} for ViT and WikiText-2 \cite{merity2016pointer} for RoBERTa), denoted as ``\textit{Random (OOD)}''.
The results show that ESM is robust to proxy composition: ViT models remain stable under both class imbalance and OOD sampling, while RoBERTa shows only mild sensitivity to class imbalance.


\subsubsection{\textbf{Impact of Eigenvalue-Weighting Power}}

Fig.~\ref{fig:eigen_power} evaluates the eigenvalue-weighting power $\gamma$. Before orthogonalization, ESM weights each parameter direction by the energy of its output shift, prioritizing higher-impact directions for preservation. This eigenvalue-weighted orthogonalization is a key improvement over the CVPR version of ESM \cite{li2026model}. Too small a $\gamma$ makes the weights nearly uniform, letting low-energy directions obscure high-energy task knowledge; too large a $\gamma$ overemphasizes dominant directions and suppresses complementary information. A moderate $\gamma=0.3$ best balances this trade-off. It yields modest ViT gains that grow with task count and an approximately $6\%$ GLUE gain with RoBERTa.

\subsubsection{\textbf{Analysis of ESM++ Design Choices}}

Table~\ref{tab:base_model} compares ESM++ across two residual ranks, two routing strategies, and two shared bases. Under prototype routing, the ESM-merged-base setting achieves higher routing accuracy and average GLUE scores than the pre-trained-base setting, suggesting that the merged representation better exposes task-specific features and provides a stronger basis for residual experts. Oracle routing uses ground-truth task identity to estimate the performance ceiling of the experts; the ESM-merged base also achieves higher oracle-routed averages. Notably, the prototype-to-oracle gap is smaller with the ESM-merged base, suggesting that prototype routing recovers a larger share of expert capacity.

\subsubsection{\textbf{Per-Layer Routing Accuracy and Task Performance}}

\begin{figure*}[t]
    \centering
    \includegraphics[width=\linewidth]{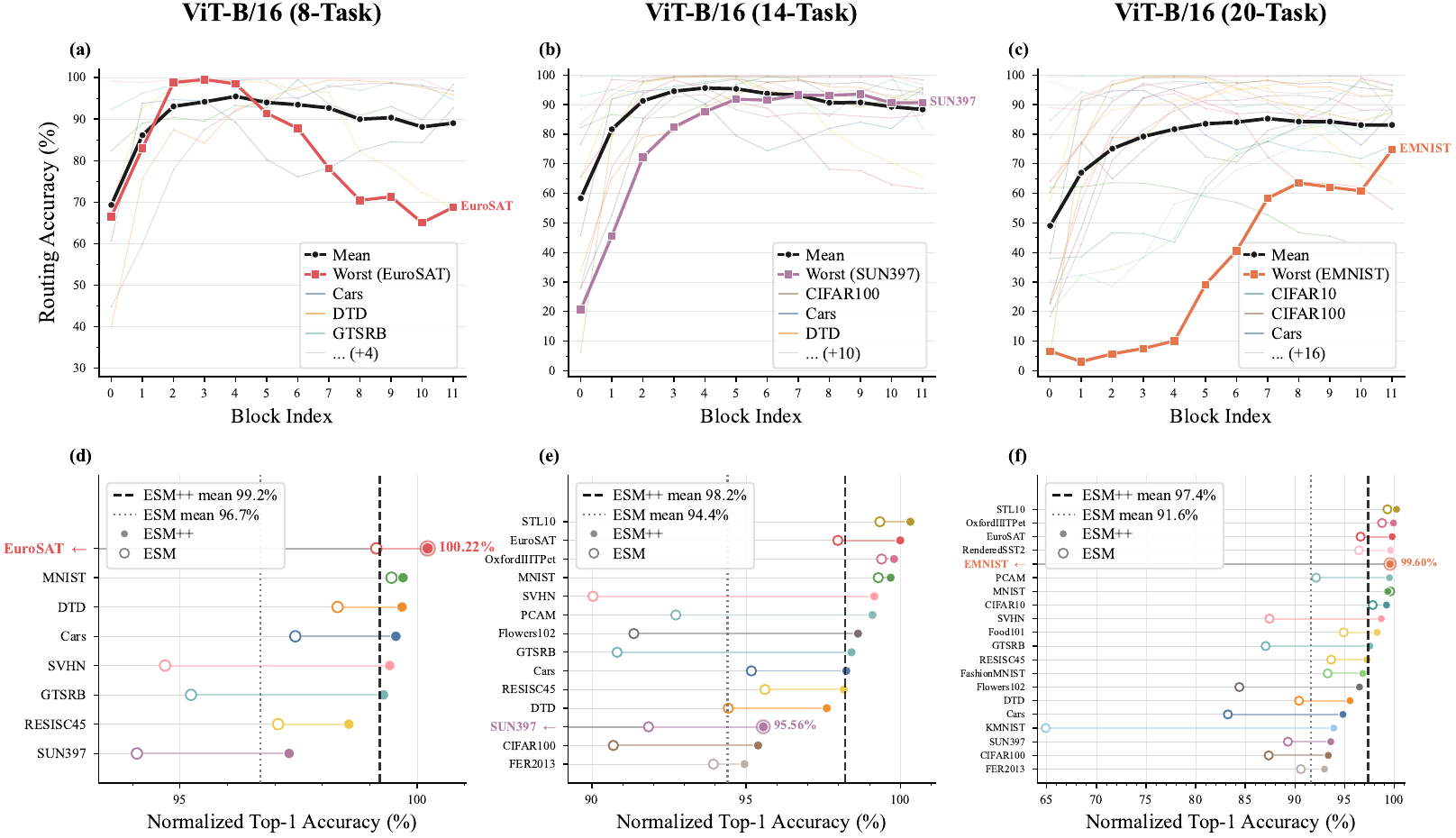}
    \caption{Per-layer routing accuracy and task-level normalized performance of ESM++ ($r=32$). The first row reports per-layer routing accuracy and highlights the task with the lowest accuracy. The second row compares each task's normalized performance with ESM, showing whether routing improves or degrades it; the task with the lowest routing accuracy in the first row is also highlighted in the second row to assess the impact of routing errors.}
    \label{fig:per_layer_routing_accuracy}
\end{figure*}

Fig.~\ref{fig:per_layer_routing_accuracy} analyzes the layer-wise routing behavior of ESM++ ($r=32$). ESM++ improves over the ESM base model on every task in the 8- and 14-task settings and on 18 of 20 tasks in the 20-task setting. The only exceptions are EMNIST and MNIST, for which performance decreases by just 0.07\% and 0.25\%, respectively. These differences are negligible and comparable to noise, especially as both tasks are already near 99.6\% normalized performance. By contrast, many other tasks show substantial gains over ESM; the three largest are KMNIST (+29.0\%), Flowers102 (+12.1\%), and Cars (+11.6\%). Overall, ESM++'s advantage over ESM grows as more tasks are merged, indicating that its benefits become more pronounced as the merging problem becomes more challenging. Moreover, even on the dataset with the lowest average routing accuracy in each setting, ESM++ outperforms the ESM base model. Together, these results demonstrate that the lightweight, training-free router is robust to imperfect task identification: even when it misses the ground-truth task, it can select a useful expert from another, related task.

\subsubsection{\textbf{Computational Overhead}}

Table~\ref{tab:overhead} reports the parameter and inference overhead of ESM and ESM++ across vision and language backbones. ESM preserves the backbone's architecture and parameter count, adding no parameters or inference-time computation. ESM++ adds per-task routers and residual experts. Our experiments show little variation in ViT inference overhead across task-count settings; accordingly, Table~\ref{tab:overhead} reports the latency averaged across these settings, approximately 5\% above ESM (1.04$\times$--1.06$\times$ across backbones and ranks). For language models, the latency overhead is 2\% on RoBERTa-Base and 8\% on Llama-3.2-3B, while added parameters remain a small fraction of each backbone. ESM++ requires no gradient-based training, using only a small unlabeled proxy set to construct prototypes and extract experts, while improving performance with modest inference overhead.

\begin{table}[t]
\centering
\caption{Parameter and inference overhead of ESM++. Router and expert parameter counts are per task. Latency is relative to ESM (batch size 256 for vision and 1 for language); vision latency is averaged over different task-count settings, and Llama latency covers prefill and 256-token decoding.}
\label{tab:overhead}
\setlength{\tabcolsep}{2.9pt}
\begin{tabular*}{\columnwidth}{@{\extracolsep{\fill}}llccc}
\toprule[1.5pt]
Backbone & Method & Router Params & Expert Params & Latency \\
\midrule[0.5pt]
\multirow{3}{*}{\textit{ViT-B/32}} & \cellcolor[HTML]{ECECEC}ESM & \multicolumn{2}{c}{\cellcolor[HTML]{ECECEC}Model Params: 87.85M} & \cellcolor[HTML]{ECECEC}1$\times$ \\
& ESM++ ($r=8$) & \multirow{2}{*}{\makecell{0.018M\\($\sim$0.02\%)}} & 0.66M ($\sim$0.8\%) & 1.04$\times$ \\
& ESM++ ($r=32$) & & 2.65M ($\sim$3.0\%) & 1.06$\times$ \\
\midrule[0.5pt]
\multirow{3}{*}{\textit{ViT-B/16}} & \cellcolor[HTML]{ECECEC}ESM & \multicolumn{2}{c}{\cellcolor[HTML]{ECECEC}Model Params: 86.19M} & \cellcolor[HTML]{ECECEC}1$\times$ \\
& ESM++ ($r=8$) & \multirow{2}{*}{\makecell{0.018M\\($\sim$0.02\%)}} & 0.66M ($\sim$0.8\%) & 1.04$\times$ \\
& ESM++ ($r=32$) & & 2.65M ($\sim$3.1\%) & 1.06$\times$ \\
\midrule[0.5pt]
\multirow{3}{*}{\textit{ViT-L/14}} & \cellcolor[HTML]{ECECEC}ESM & \multicolumn{2}{c}{\cellcolor[HTML]{ECECEC}Model Params: 303.97M} & \cellcolor[HTML]{ECECEC}1$\times$ \\
& ESM++ ($r=8$) & \multirow{2}{*}{\makecell{0.049M\\($\sim$0.02\%)}} & 1.77M ($\sim$0.6\%) & 1.04$\times$ \\
& ESM++ ($r=32$) & & 7.08M ($\sim$2.3\%) & 1.05$\times$ \\
\midrule[0.5pt]
\multirow{2}{*}{\makecell[c]{\textit{RoBERTa}\\\textit{(Base)}}} & \cellcolor[HTML]{ECECEC}ESM & \multicolumn{2}{c}{\cellcolor[HTML]{ECECEC}Model Params: 124.06M} & \cellcolor[HTML]{ECECEC}1$\times$ \\
& ESM++ ($r=32$) & 0.66M ($\sim$0.5\%) & 5.31M ($\sim$4.3\%) & 1.02$\times$ \\
\midrule[0.5pt]
\multirow{2}{*}{\makecell[c]{\textit{Llama}\\\textit{(3.2-3B)}}} & \cellcolor[HTML]{ECECEC}ESM & \multicolumn{2}{c}{\cellcolor[HTML]{ECECEC}Model Params: 3.21B} & \cellcolor[HTML]{ECECEC}1$\times$ \\
& ESM++ ($r=64$) & 5.96M ($\sim$0.2\%) & 97.26M ($\sim$3.0\%) & 1.08$\times$ \\
\bottomrule[1.5pt]
\end{tabular*}
\end{table}

\section{Conclusion}
In this paper, we studied model merging from the perspective of output activation shifts induced by task-specific updates. We showed that these shifts concentrate in a few principal directions that better reflect functional changes than parameter-space decomposition, while accumulated low-energy directions can lead to merging interference. Motivated by this observation, we proposed Essential Subspace Decomposition (ESD) to preserve essential update components, and developed ESM for compact static fusion and ESM++ for dynamic low-rank residual routing. Extensive experiments on vision and language benchmarks demonstrate strong performance with modest per-expert parameter overhead, providing a principled approach to structured model merging.

\textbf{Limitations and Future Work.} Despite its effectiveness, the current method is mainly designed for merging models that share the same architecture and originate from the same base model, where task updates can be directly compared and composed in a common parameter and activation space. Extending this framework to more general scenarios, such as merging models from different sources, training recipes, or architectures, remains an important direction for future work.


\bibliographystyle{IEEEtran}
\bibliography{main}

@String(CVPR = {CVPR})

@String(ICCV = {ICCV})

@String(ECCV = {ECCV})

@String(NeurIPS = {NeurIPS})

@String(ICLR = {ICLR})

@String(AAAI = {AAAI})

@inproceedings{ilharcoediting,
  title={Editing models with task arithmetic},
  author={Ilharco, Gabriel and Ribeiro, Marco Tulio and Wortsman, Mitchell and Schmidt, Ludwig and Hajishirzi, Hannaneh and Farhadi, Ali},
  booktitle={ICLR},
  year = 2023
}

@inproceedings{gargiulo2025task,
  title={Task singular vectors: Reducing task interference in model merging},
  author={Gargiulo, Antonio Andrea and Crisostomi, Donato and Bucarelli, Maria Sofia and Scardapane, Simone and Silvestri, Fabrizio and Rodola, Emanuele},
  booktitle={CVPR},
  year={2025}
}

@inproceedings{marczakno,
  title={No Task Left Behind: Isotropic Model Merging with Common and Task-Specific Subspaces},
  author={Marczak, Daniel and Magistri, Simone and Cygert, Sebastian and Twardowski, Bart{\l}omiej and Bagdanov, Andrew D and van de Weijer, Joost},
  booktitle={ICML},
  year={2025}
}

@inproceedings{wortsman2022model,
  title={Model soups: averaging weights of multiple fine-tuned models improves accuracy without increasing inference time},
  author={Wortsman, Mitchell and Ilharco, Gabriel and Gadre, Samir Ya and Roelofs, Rebecca and Gontijo-Lopes, Raphael and Morcos, Ari S and Namkoong, Hongseok and Farhadi, Ali and Carmon, Yair and Kornblith, Simon and others},
  booktitle={ICML},
  year={2022},
}

@inproceedings{yadav2023ties,
  title={Ties-merging: Resolving interference when merging models},
  author={Yadav, Prateek and Tam, Derek and Choshen, Leshem and Raffel, Colin A and Bansal, Mohit},
  booktitle={NeurIPS},
  year={2023}
}

@inproceedings{yangadamerging,
  title={AdaMerging: Adaptive Model Merging for Multi-Task Learning},
  author={Yang, Enneng and Wang, Zhenyi and Shen, Li and Liu, Shiwei and Guo, Guibing and Wang, Xingwei and Tao, Dacheng},
  booktitle={ICLR},
  year={2024}
}

@inproceedings{yu2024language,
  title={Language models are super mario: Absorbing abilities from homologous models as a free lunch},
  author={Yu, Le and Yu, Bowen and Yu, Haiyang and Huang, Fei and Li, Yongbin},
  booktitle={ICML},
  year={2024}
}

@inproceedings{matena2022merging,
  title={Merging models with fisher-weighted averaging},
  author={Matena, Michael S and Raffel, Colin A},
  booktitle={NeurIPS},
  year={2022}
}

@inproceedings{jindataless,
  title={Dataless Knowledge Fusion by Merging Weights of Language Models},
  author={Jin, Xisen and Ren, Xiang and Preotiuc-Pietro, Daniel and Cheng, Pengxiang},
  booktitle={ICLR},
  year={2023}
}

@inproceedings{hulora,
  title={LoRA: Low-Rank Adaptation of Large Language Models},
  author={Hu, Edward J and Wallis, Phillip and Allen-Zhu, Zeyuan and Li, Yuanzhi and Wang, Shean and Wang, Lu and Chen, Weizhu and others},
  booktitle={ICLR},
  year={2022}
}

@inproceedings{wangsvd,
  title={SVD-LLM: Truncation-aware Singular Value Decomposition for Large Language Model Compression},
  author={Wang, Xin and Zheng, Yu and Wan, Zhongwei and Zhang, Mi},
  booktitle={ICLR},
  year={2025}
}

@inproceedings{wanglocalizing,
  title={Localizing Task Information for Improved Model Merging and Compression},
  author={Wang, Ke and Dimitriadis, Nikolaos and Ortiz-Jimenez, Guillermo and Fleuret, Fran{\c{c}}ois and Frossard, Pascal},
  booktitle={ICML},
  year={2024}
}

@inproceedings{radford2021learning,
  title={Learning transferable visual models from natural language supervision},
  author={Radford, Alec and Kim, Jong Wook and Hallacy, Chris and Ramesh, Aditya and Goh, Gabriel and Agarwal, Sandhini and Sastry, Girish and Askell, Amanda and Mishkin, Pamela and Clark, Jack and others},
  booktitle={ICML},
  year={2021},
}

@inproceedings{krause20133d,
  title={3d object representations for fine-grained categorization},
  author={Krause, Jonathan and Stark, Michael and Deng, Jia and Fei-Fei, Li},
  booktitle={ICCV Workshops},
  year={2013}
}

@inproceedings{cimpoi2014describing,
  title={Describing textures in the wild},
  author={Cimpoi, Mircea and Maji, Subhransu and Kokkinos, Iasonas and Mohamed, Sammy and Vedaldi, Andrea},
  booktitle={CVPR},
  year={2014}
}

@article{helber2019eurosat,
  title={Eurosat: A novel dataset and deep learning benchmark for land use and land cover classification},
  author={Helber, Patrick and Bischke, Benjamin and Dengel, Andreas and Borth, Damian},
  journal={IEEE J. Sel. Top. Appl. Earth Obs. Remote Sens.},
  year={2019},
}

@inproceedings{stallkamp2011german,
  title={The German traffic sign recognition benchmark: a multi-class classification competition},
  author={Stallkamp, Johannes and Schlipsing, Marc and Salmen, Jan and Igel, Christian},
  booktitle={IJCNN},
  year={2011},
}

@article{lecun2002gradient,
  title={Gradient-based learning applied to document recognition},
  author={LeCun, Yann and Bottou, L{\'e}on and Bengio, Yoshua and Haffner, Patrick},
  journal={Proc. IEEE},
  year={2002},
}

@article{cheng2017remote,
  title={Remote sensing image scene classification: Benchmark and state of the art},
  author={Cheng, Gong and Han, Junwei and Lu, Xiaoqiang},
  journal={Proc. IEEE},
  year={2017},
}

@article{xiao2016sun,
  title={Sun database: Exploring a large collection of scene categories},
  author={Xiao, Jianxiong and Ehinger, Krista A and Hays, James and Torralba, Antonio and Oliva, Aude},
  journal={Int. J. Comput. Vis.},
  year={2016},
}

@inproceedings{netzer2011reading,
  title={Reading digits in natural images with unsupervised feature learning},
  author={Netzer, Yuval and Wang, Tao and Coates, Adam and Bissacco, Alessandro and Wu, Baolin and Ng, Andrew Y and others},
  booktitle={NeurIPS Workshops},
  year={2011}
}

@techreport{krizhevsky2009learning,
  title={Learning multiple layers of features from tiny images},
  author={Krizhevsky, Alex and Hinton, Geoffrey and others},
  year={2009},
  Institution = {Toronto, ON, Canada}
}

@inproceedings{coates2011analysis,
  title={An analysis of single-layer networks in unsupervised feature learning},
  author={Coates, Adam and Ng, Andrew and Lee, Honglak},
  booktitle={AISTATS},
  year={2011},
}

@inproceedings{nilsback2008automated,
  title={Automated flower classification over a large number of classes},
  author={Nilsback, Maria-Elena and Zisserman, Andrew},
  booktitle={ICVGIP},
  year={2008},
}

@inproceedings{parkhi2012cats,
  title={Cats and dogs},
  author={Parkhi, Omkar M and Vedaldi, Andrea and Zisserman, Andrew and Jawahar, CV},
  booktitle={CVPR},
  year={2012},
}

@inproceedings{veeling2018rotation,
  title={Rotation equivariant CNNs for digital pathology},
  author={Veeling, Bastiaan S and Linmans, Jasper and Winkens, Jim and Cohen, Taco and Welling, Max},
  booktitle={MICCAI},
  year={2018},
}

@inproceedings{goodfellow2013challenges,
  title={Challenges in representation learning: A report on three machine learning contests},
  author={Goodfellow, Ian J and Erhan, Dumitru and Carrier, Pierre Luc and Courville, Aaron and Mirza, Mehdi and Hamner, Ben and Cukierski, Will and Tang, Yichuan and Thaler, David and Lee, Dong-Hyun and others},
  booktitle={ICONIP},
  year={2013},
}

@inproceedings{cohen2017emnist,
  title={EMNIST: Extending MNIST to handwritten letters},
  author={Cohen, Gregory and Afshar, Saeed and Tapson, Jonathan and Van Schaik, Andre},
  booktitle={IJCNN},
  year={2017},
}

@inproceedings{bossard2014food,
  title={Food-101--mining discriminative components with random forests},
  author={Bossard, Lukas and Guillaumin, Matthieu and Van Gool, Luc},
  booktitle={ECCV},
  year={2014},
}

@article{xiao2017fashion,
  title={Fashion-mnist: a novel image dataset for benchmarking machine learning algorithms},
  author={Xiao, Han and Rasul, Kashif and Vollgraf, Roland},
  journal={arXiv preprint arXiv:1708.07747},
  year={2017}
}

@inproceedings{socher2013recursive,
  title={Recursive deep models for semantic compositionality over a sentiment treebank},
  author={Socher, Richard and Perelygin, Alex and Wu, Jean and Chuang, Jason and Manning, Christopher D and Ng, Andrew Y and Potts, Christopher},
  booktitle={EMNLP},
  year={2013}
}

@article{clanuwat2018deep,
  title={Deep learning for classical japanese literature},
  author={Clanuwat, Tarin and Bober-Irizar, Mikel and Kitamoto, Asanobu and Lamb, Alex and Yamamoto, Kazuaki and Ha, David},
  journal={arXiv preprint arXiv:1812.01718},
  year={2018}
}

@article{russakovsky2015imagenet,
  title={Imagenet large scale visual recognition challenge},
  author={Russakovsky, Olga and Deng, Jia and Su, Hao and Krause, Jonathan and Satheesh, Sanjeev and Ma, Sean and Huang, Zhiheng and Karpathy, Andrej and Khosla, Aditya and Bernstein, Michael and others},
  journal={Int. J. Comput. Vis.},
  year={2015},
}

@inproceedings{stoicamodel,
  title={Model merging with SVD to tie the Knots},
  author={Stoica, George and Ramesh, Pratik and Ecsedi, Boglarka and Choshen, Leshem and Hoffman, Judy},
  booktitle={ICLR},
  year={2025},
}

@inproceedings{du2024parameter,
  title={Parameter competition balancing for model merging},
  author={Du, Guodong and Lee, Junlin and Li, Jing and Jiang, Runhua and Guo, Yifei and Yu, Shuyang and Liu, Hanting and Goh, Sim K and Tang, Ho-Kin and He, Daojing and others},
  booktitle={NeurIPS},
  year={2024}
}

@inproceedings{zhang2024knowledge,
  title={Knowledge composition using task vectors with learned anisotropic scaling},
  author={Zhang, Frederic Z and Albert, Paul and Rodriguez-Opazo, Cristian and van den Hengel, Anton and Abbasnejad, Ehsan},
  booktitle={NeurIPS},
  year={2024}
}

@inproceedings{wanglines,
  title={LiNeS: Post-training Layer Scaling Prevents Forgetting and Enhances Model Merging},
  author={Wang, Ke and Dimitriadis, Nikolaos and Favero, Alessandro and Ortiz-Jimenez, Guillermo and Fleuret, Fran{\c{c}}ois and Frossard, Pascal},
  booktitle={ICLR},
  year={2025}
}

@inproceedings{chengwhoever,
  title={Whoever Started the interference Should End It: Guiding Data-Free Model Merging via Task Vectors},
  author={Cheng, Runxi and Xiong, Feng and Wei, Yongxian and Zhu, Wanyun and Yuan, Chun},
  booktitle={ICML},
  year={2025}
}

@inproceedings{huang2024emr,
  title={Emr-merging: Tuning-free high-performance model merging},
  author={Huang, Chenyu and Ye, Peng and Chen, Tao and He, Tong and Yue, Xiangyu and Ouyang, Wanli},
  booktitle={NeurIPS},
  year={2024}
}

@inproceedings{sun2022singular,
  title={Singular value fine-tuning: Few-shot segmentation requires few-parameters fine-tuning},
  author={Sun, Yanpeng and Chen, Qiang and He, Xiangyu and Wang, Jian and Feng, Haocheng and Han, Junyu and Ding, Errui and Cheng, Jian and Li, Zechao and Wang, Jingdong},
  booktitle={NeurIPS},
  year={2022}
}

@inproceedings{han2023svdiff,
  title={Svdiff: Compact parameter space for diffusion fine-tuning},
  author={Han, Ligong and Li, Yinxiao and Zhang, Han and Milanfar, Peyman and Metaxas, Dimitris and Yang, Feng},
  booktitle={ICCV},
  year={2023}
}

@inproceedings{wang2025svd,
  title={SVD-LLM V2: Optimizing Singular Value Truncation for Large Language Model Compression},
  author={Wang, Xin and Alam, Samiul and Wan, Zhongwei and Shen, Hui and Zhang, Mi},
  booktitle={NAACL},
  year={2025}
}

@inproceedings{li2023losparse,
  title={Losparse: Structured compression of large language models based on low-rank and sparse approximation},
  author={Li, Yixiao and Yu, Yifan and Zhang, Qingru and Liang, Chen and He, Pengcheng and Chen, Weizhu and Zhao, Tuo},
  booktitle={ICML},
  year={2023},
}

@inproceedings{saha2023matrix,
  title={Matrix compression via randomized low rank and low precision factorization},
  author={Saha, Rajarshi and Srivastava, Varun and Pilanci, Mert},
  booktitle={NeurIPS},
  year={2023}
}

@inproceedings{dettmers2023qlora,
  title={Qlora: Efficient finetuning of quantized llms},
  author={Dettmers, Tim and Pagnoni, Artidoro and Holtzman, Ari and Zettlemoyer, Luke},
  booktitle={NeurIPS},
  year={2023}
}

@article{ding2023parameter,
  title={Parameter-efficient fine-tuning of large-scale pre-trained language models},
  author={Ding, Ning and Qin, Yujia and Yang, Guang and Wei, Fuchao and Yang, Zonghan and Su, Yusheng and Hu, Shengding and Chen, Yulin and Chan, Chi-Min and Chen, Weize and others},
  journal={Nat. Mach. Intell.},
  year={2023},
}

@inproceedings{kornblith2019similarity,
  title={Similarity of neural network representations revisited},
  author={Kornblith, Simon and Norouzi, Mohammad and Lee, Honglak and Hinton, Geoffrey},
  booktitle={ICML},
  year={2019},
}

@inproceedings{yan2025calm,
  title={CALM: Consensus-Aware Localized Merging for Multi-Task Learning},
  author={Yan, Kunda and Zhang, Min and Cui, Sen and Zikun, Qu and Jiang, Bo and Liu, Feng and Zhang, Changshui},
  booktitle={ICML},
  year={2025}
}

@inproceedings{sun2025towards,
  title={Towards Minimizing Feature Drift in Model Merging: Layer-wise Task Vector Fusion for Adaptive Knowledge Integration},
  author={Sun, Wenju and Li, Qingyong and Wang, Wen and Liu, Yang and Geng, Yangliao and Li, Boyang},
  booktitle={NeurIPS},
  year={2025}
}

@inproceedings{zhang2026dc,
  title={DC-Merge: Improving Model Merging with Directional Consistency},
  author={Zhang, Han-Chen and Zhou, Zi-Hao and Luo, Mao-Lin and Di, Shimin and Zhang, Min-Ling and Wei, Tong},
  booktitle={CVPR},
  year={2026}
}

@article{he2024localize,
  title={Localize-and-Stitch: Efficient Model Merging via Sparse Task Arithmetic},
  author={He, Yifei and Hu, Yuzheng and Lin, Yong and Zhang, Tong and Zhao, Han},
  journal={TMLR},
  year={2024}
}

@inproceedings{li2026model,
  title={Model merging in the essential subspace},
  author={Li, Longhua and Qi, Lei and Tian, Qi and Geng, Xin},
  booktitle={CVPR},
  year={2026}
}

@article{li2026improving,
  title={Improving Model Fusion by Training-time Neuron Alignment with Fixed Neuron Anchors},
  author={Li, Zexi and Li, Zhiqi and Lin, Jie and Shen, Tao and Xiao, Jun and Guo, Yike and Lin, Tao and Wu, Chao},
  journal={IEEE Trans. Pattern Anal. Mach. Intell.},
  year={2026}
}

@article{singh2020model,
  title={Model fusion via optimal transport},
  author={Singh, Sidak Pal and Jaggi, Martin},
  journal={NeurIPS},
  volume={33},
  pages={22045--22055},
  year={2020}
}

@article{tatro2020optimizing,
  title={Optimizing mode connectivity via neuron alignment},
  author={Tatro, Norman and Chen, Pin-Yu and Das, Payel and Melnyk, Igor and Sattigeri, Prasanna and Lai, Rongjie},
  journal={NeurIPS},
  volume={33},
  pages={15300--15311},
  year={2020}
}

@inproceedings{pena2023re,
  title={Re-basin via implicit sinkhorn differentiation},
  author={Pe{\~n}a, Fidel A Guerrero and Medeiros, Heitor Rapela and Dubail, Thomas and Aminbeidokhti, Masih and Granger, Eric and Pedersoli, Marco},
  booktitle={CVPR},
  pages={20237--20246},
  year={2023}
}

@article{merity2016pointer,
  title={Pointer sentinel mixture models},
  author={Merity, Stephen and Xiong, Caiming and Bradbury, James and Socher, Richard},
  journal={arXiv preprint arXiv:1609.07843},
  year={2016}
}

@inproceedings{sun2025cat,
  title={CAT Merging: A Training-Free Approach for Resolving Conflicts in Model Merging},
  author={Sun, Wenju and Li, Qingyong and Geng, Yangliao and Li, Boyang},
  booktitle={ICML},
  pages={57523--57543},
  year={2025},
  organization={PMLR}
}

@article{li2026energy,
  title={Energy-Structured Low-Rank Adaptation for Continual Learning},
  author={Li, Longhua and Qi, Lei and Tian, Qi and Geng, Xin},
  journal={arXiv preprint arXiv:2605.27482},
  year={2026}
}

@inproceedings{he2025mergebench,
  title={MergeBench: A Benchmark for Merging Domain-Specialized {LLM}s},
  author={He, Yifei and Zeng, Siqi and Hu, Yuzheng and Yang, Rui and Zhang, Tong and Zhao, Han},
  booktitle={NeurIPS Datasets and Benchmarks Track},
  volume={38},
  year={2025}
}

@inproceedings{wang2018glue,
  title={GLUE: A multi-task benchmark and analysis platform for natural language understanding},
  author={Wang, Alex and Singh, Amanpreet and Michael, Julian and Hill, Felix and Levy, Omer and Bowman, Samuel},
  booktitle={EMNLP Workshop},
  pages={353--355},
  year={2018}
}

@article{liu2019roberta,
  title={Roberta: A robustly optimized bert pretraining approach},
  author={Liu, Yinhan and Ott, Myle and Goyal, Naman and Du, Jingfei and Joshi, Mandar and Chen, Danqi and Levy, Omer and Lewis, Mike and Zettlemoyer, Luke and Stoyanov, Veselin},
  journal={arXiv preprint arXiv:1907.11692},
  year={2019}
}

@inproceedings{li2026stratified,
  title={Stratified Knowledge-Density Super-Network for Scalable Vision Transformers},
  author={Li, Longhua and Qi, Lei and Geng, Xin},
  booktitle={AAAI},
  volume={40},
  number={27},
  pages={22985--22993},
  year={2026}
}

@article{grattafiori2024llama,
  title={The llama 3 herd of models},
  author={Grattafiori, Aaron and Dubey, Abhimanyu and Jauhri, Abhinav and Pandey, Abhinav and Kadian, Abhishek and Al-Dahle, Ahmad and Letman, Aiesha and Mathur, Akhil and Schelten, Alan and Vaughan, Alex and others},
  journal={arXiv preprint arXiv:2407.21783},
  year={2024}
}

@article{zhou2023instruction,
  title={Instruction-following evaluation for large language models},
  author={Zhou, Jeffrey and Lu, Tianjian and Mishra, Swaroop and Brahma, Siddhartha and Basu, Sujoy and Luan, Yi and Zhou, Denny and Hou, Le},
  journal={arXiv preprint arXiv:2311.07911},
  year={2023}
}

@article{cobbe2021training,
  title={Training verifiers to solve math word problems},
  author={Cobbe, Karl and Kosaraju, Vineet and Bavarian, Mohammad and Chen, Mark and Jun, Heewoo and Kaiser, Lukasz and Plappert, Matthias and Tworek, Jerry and Hilton, Jacob and Nakano, Reiichiro and others},
  journal={arXiv preprint arXiv:2110.14168},
  year={2021}
}

@inproceedings{lai2023okapi,
  title={Okapi: Instruction-tuned large language models in multiple languages with reinforcement learning from human feedback},
  author={Lai, Viet and Nguyen, Chien and Ngo, Nghia and Nguy{\~{e}}n, Thu{\d{a}}t and Dernoncourt, Franck and Rossi, Ryan and Nguyen, Thien},
  booktitle={EMNLP},
  pages={318--327},
  year={2023}
}

@article{chen2021evaluating,
  title={Evaluating large language models trained on code},
  author={Chen, Mark and Tworek, Jerry and Jun, Heewoo and Yuan, Qiming and Pinto, Henrique Ponde De Oliveira and Kaplan, Jared and Edwards, Harri and Burda, Yuri and Joseph, Nicholas and Brockman, Greg and others},
  journal={arXiv preprint arXiv:2107.03374},
  year={2021}
}

@article{austin2021program,
  title={Program synthesis with large language models},
  author={Austin, Jacob and Odena, Augustus and Nye, Maxwell and Bosma, Maarten and Michalewski, Henryk and Dohan, David and Jiang, Ellen and Cai, Carrie and Terry, Michael and Le, Quoc and others},
  journal={arXiv preprint arXiv:2108.07732},
  year={2021}
}

@article{han2024wildguard,
  title={Wildguard: Open one-stop moderation tools for safety risks, jailbreaks, and refusals of llms},
  author={Han, Seungju and Rao, Kavel and Ettinger, Allyson and Jiang, Liwei and Lin, Bill Yuchen and Lambert, Nathan and Choi, Yejin and Dziri, Nouha},
  journal={NeurIPS},
  volume={37},
  pages={8093--8131},
  year={2024}
}

@article{mazeika2024harmbench,
  title={Harmbench: A standardized evaluation framework for automated red teaming and robust refusal},
  author={Mazeika, Mantas and Phan, Long and Yin, Xuwang and Zou, Andy and Wang, Zifan and Mu, Norman and Sakhaee, Elham and Li, Nathaniel and Basart, Steven and Li, Bo and others},
  journal={arXiv preprint arXiv:2402.04249},
  year={2024}
}

@inproceedings{shen2024anything,
  title={" do anything now": Characterizing and evaluating in-the-wild jailbreak prompts on large language models},
  author={Shen, Xinyue and Chen, Zeyuan and Backes, Michael and Shen, Yun and Zhang, Yang},
  booktitle={ACM CCS},
  pages={1671--1685},
  year={2024}
}

@inproceedings{rottger2024xstest,
  title={Xstest: A test suite for identifying exaggerated safety behaviours in large language models},
  author={R{\"o}ttger, Paul and Kirk, Hannah and Vidgen, Bertie and Attanasio, Giuseppe and Bianchi, Federico and Hovy, Dirk},
  booktitle={NAACL},
  pages={5377--5400},
  year={2024}
}

@article{chen2024privfusion,
  title={{PrivFusion}: Privacy-Preserving Model Fusion via Decentralized Federated Graph Matching},
  author={Chen, Qian and Chen, Yiqiang and Jiang, Xinlong and Zhang, Teng and Dai, Weiwei and Huang, Wuliang and Yan, Bingjie and Yan, Zhen and Lu, Wang and Ye, Bo},
  journal={IEEE Trans. Knowl. Data Eng.},
  volume={36},
  number={12},
  pages={9051--9064},
  year={2024}
}

@article{chen2026rethinking,
  title={Rethinking Model Merging: Partial Merging Via Conflict Mitigation and Adaptive Weighting},
  author={Chen, H. and Wang, T. and Li, Y. and Lim, W. Y. B.},
  journal={IEEE Trans. Knowl. Data Eng.},
  year={2026}
}

@article{lin2026exme,
  title={{ExMe}: Keep Improving With Extrapolation and Merging},
  author={Lin, Yiguan and Xu, Bin and Li, Yinghao and Gao, Yang},
  journal={IEEE Trans. Knowl. Data Eng.},
  volume={38},
  number={9},
  pages={5867--5878},
  year={2026}
}

@article{li2024map,
  title={{MAP}: Model Aggregation and Personalization in Federated Learning With Incomplete Classes},
  author={Li, Xin-Chun and Song, Shaoming and Li, Yinchuan and Li, Bingshuai and Shao, Yunfeng and Yang, Yang and Zhan, De-Chuan},
  journal={IEEE Trans. Knowl. Data Eng.},
  volume={36},
  number={11},
  pages={6560--6573},
  year={2024}
}

\clearpage
\onecolumn
\clearpage
\setcounter{page}{1}
\maketitlesupplementary

\appendices
\newtheorem{propositionplain}{Proposition}
\newenvironment{proposition}{%
  \begin{theorembox}\begin{propositionplain}%
}{%
  \end{propositionplain}\end{theorembox}%
}

\section{Proofs}

This section provides the derivations of the expected output error after truncation for both the standard SVD and the proposed Essential Subspace Decomposition (ESD), as well as the connection between orthogonal polar factors and the Orthogonal Procrustes solution.

\subsection{Proof for SVD Truncation Error}
\label{sec:Proof_SVD_Truncation_Error}

\begin{proposition}
Given a task matrix $\Delta W \in \mathbb{R}^{d_{\text{out}} \times d_{\text{in}}}$ with singular value decomposition $\Delta W = U\Sigma V^\top = \sum_{i=1}^s \sigma_i u_i v_i^\top$, where $s=\operatorname{rank}(\Delta W)$. Let $\widehat{\Delta W} = \sum_{i=1}^r \sigma_i u_i v_i^\top$ be its top-$r$ truncated approximation. For an input $x$ drawn from a distribution $\mathcal{D}$, the expected squared $L_2$ error on the output activation is
$$ \mathbb{E}_{x\sim\mathcal{D}} \left[ \|\Delta W x - \widehat{\Delta W} x \|_2^2 \right] = \sum_{i=r+1}^s\sigma_{i}^{2} \cdot \mathbb{E}_{x\sim\mathcal{D}} \left[ (v_{i}^{\top}x)^2 \right].
$$
\end{proposition}

\begin{proof}
The error matrix resulting from the truncation is the sum of the discarded components:
\begin{equation}
\Delta W - \widehat{\Delta W} = \sum_{i=r+1}^{s} \sigma_i u_i v_i^\top.
\end{equation}
The error on the output activation for a given input $x$ is given by:
\begin{equation}
(\Delta W - \widehat{\Delta W})x = \left( \sum_{i=r+1}^{s} \sigma_i u_i v_i^\top \right) x = \sum_{i=r+1}^{s} \sigma_i u_i (v_i^\top x).
\end{equation}
Since $v_i^\top x$ is a scalar, we can rewrite this as a linear combination of the orthonormal vectors $u_i$. The squared $L_2$ norm of this error vector is given by:
\begin{equation}
\|(\Delta W - \widehat{\Delta W})x\|_2^2 = \left\| \sum_{i=r+1}^{s} (\sigma_i v_i^\top x) u_i \right\|_2^2.
\end{equation}
Because the left singular vectors $\{u_i\}$ form an orthonormal set, the squared norm of their weighted sum is the sum of the squares of the weights:
\begin{equation}
\|(\Delta W - \widehat{\Delta W})x\|_2^2 = \sum_{i=r+1}^{s} (\sigma_i v_i^\top x)^2 = \sum_{i=r+1}^{s} \sigma_i^2 (v_i^\top x)^2.
\end{equation}
By taking the expectation over the input distribution $\mathcal{D}$ and applying the linearity of expectation, we arrive at the final expression:
\begin{equation}
\mathbb{E}_{x\sim\mathcal{D}} \left[ \|\Delta W x - \widehat{\Delta W} x \|_2^2 \right] = \sum_{i=r+1}^s\sigma_{i}^{2} \cdot \mathbb{E}_{x\sim\mathcal{D}} \left[ (v_{i}^{\top}x)^2 \right].
\end{equation}
This completes the proof.
\end{proof}

\subsection{Proof of ESD Truncation Error and Optimality}
\label{sec:Proof_ESD_Truncation_Error}

\begin{proposition}
Given a task matrix $\Delta W \in \mathbb{R}^{d_{\text{out}} \times d_{\text{in}}}$ and the activation shift $y=\Delta W x$, let $M_y=\mathbb{E}_{x\sim\mathcal{D}}[yy^\top]$ be the uncentered second-moment matrix of output shifts. Let $\{e_i\}_{i=1}^{d_{\text{out}}}$ be the eigenvectors of $M_y$, with eigenvalues $\lambda_1\geq\cdots\geq\lambda_{d_{\text{out}}}\geq0$. Let $\hat{E}=[e_1,\dots,e_r]$ and $\widehat{\Delta W}=\hat{E}\hat{C}=\hat{E}(\hat{E}^\top\Delta W)$ be the ESD reconstruction. Then
$$ \mathbb{E}_{x\sim\mathcal{D}} \left[ \|\Delta W x - \widehat{\Delta W} x \|_2^2 \right] = \sum_{i=r+1}^{d_{\text{out}}} \lambda_i. $$
Moreover, $\widehat{\Delta W}$ minimizes this expected error over all matrices of rank at most $r$. Consequently, its error is no greater than that of the rank-$r$ approximation obtained by applying SVD directly to $\Delta W$. For empirical ESD, the same statements hold with $\mathcal{D}$ replaced by the empirical proxy distribution; equivalently, $\{e_i\}$ are the right singular vectors of $\Delta O=X_{\text{proxy}}\Delta W^\top$, and $\lambda_i$ equal the corresponding squared singular values divided by the number of recorded proxy features.
\end{proposition}

\begin{proof}
The error on the output activation for an input $x$ is given by:
\begin{equation}
\Delta W x - \widehat{\Delta W} x = \Delta W x - \hat{E} \hat{E}^\top \Delta W x = (I - \hat{E} \hat{E}^\top) \Delta W x.
\end{equation}
The matrix $(I - \hat{E} \hat{E}^\top)$ is the projection matrix onto the subspace spanned by the discarded directions $\{e_{r+1}, \dots, e_{d_{\text{out}}}\}$. Let $y = \Delta W x$ be the activation shift for input $x$. The error vector can be expressed as the projection of $y$ onto this orthogonal subspace:
\begin{equation}
(I - \hat{E} \hat{E}^\top) y = \sum_{i=r+1}^{d_{\text{out}}} (e_i^\top y) e_i.
\end{equation}
Since $\{e_i\}$ form an orthonormal basis, the squared $L_2$ norm is the sum of the squares of the projection coefficients:
\begin{equation}
\begin{aligned}
\|\Delta W x - \widehat{\Delta W} x \|_2^2 &= \left\| \sum_{i=r+1}^{d_{\text{out}}} (e_i^\top y) e_i \right\|_2^2 \\
&= \sum_{i=r+1}^{d_{\text{out}}} (e_i^\top y)^2.
\end{aligned}
\end{equation}
Taking the expectation over $\mathcal{D}$ gives
\begin{equation}
\begin{aligned}
\mathbb{E}_{x\sim\mathcal{D}} \left[ \|\Delta W x - \widehat{\Delta W} x \|_2^2 \right]
&= \sum_{i=r+1}^{d_{\text{out}}} \mathbb{E}_{x\sim\mathcal{D}} \left[ (e_i^\top y)^2 \right] \\
&= \sum_{i=r+1}^{d_{\text{out}}} e_i^\top \mathbb{E}_{x\sim\mathcal{D}}[yy^\top] e_i.
\end{aligned}
\end{equation}
By the definition of $M_y$, each retained or discarded direction satisfies
\begin{equation}
M_y e_i = \lambda_i e_i,
\end{equation}
and hence
\begin{equation}
e_i^\top M_y e_i = \lambda_i.
\end{equation}
Substituting this identity into the expected error yields
\begin{equation}
\mathbb{E}_{x\sim\mathcal{D}} \left[ \|\Delta W x - \widehat{\Delta W} x \|_2^2 \right] = \sum_{i=r+1}^{d_{\text{out}}} \lambda_i.
\end{equation}
It remains to establish optimality. Let $\Sigma_x=\mathbb{E}[xx^\top]$ and define $B=\Delta W\Sigma_x^{1/2}$. For any candidate matrix $A$,
\begin{equation}
\mathbb{E}_{x\sim\mathcal{D}}\|\Delta W x-Ax\|_2^2
=\|(\Delta W-A)\Sigma_x^{1/2}\|_F^2
=\|B-A\Sigma_x^{1/2}\|_F^2.
\end{equation}
Moreover, $M_y=BB^\top$, so the eigenvectors $e_i$ are the left singular vectors of $B$. Therefore,
\begin{equation}
\widehat{\Delta W}\Sigma_x^{1/2}=\hat{E}\hat{E}^\top B
\end{equation}
is the best rank-$r$ approximation of $B$ by the Eckart--Young theorem. For every matrix $A$ with $\operatorname{rank}(A)\leq r$, the matrix $A\Sigma_x^{1/2}$ also has rank at most $r$ and hence cannot achieve a smaller error. Because the rank-$r$ parameter-space SVD reconstruction is one such candidate $A$, ESD incurs no greater expected output error on the same distribution. Empirically, if $\Delta O=U\Sigma V^\top$ is the SVD of the uncentered activation-shift matrix, then the columns of $V$ are the eigenvectors of $\Delta O^\top\Delta O$ and the same argument applies to the empirical proxy distribution. This completes the proof.
\end{proof}

\subsection{Orthogonal Polar Factor as an Orthogonal Procrustes Solution}
\label{sec:whitening_procrustes}
\begin{proposition}
Let $X=U\Sigma V^\top$ be the thin SVD of a rectangular matrix with full rank in its smaller dimension. If $X$ is tall, its polar factor is $UV^\top=X(X^\top X)^{-1/2}$ and has orthonormal columns. If $X$ is wide, its polar factor is $UV^\top=(XX^\top)^{-1/2}X$ and has orthonormal rows. In either case, $UV^\top$ is the Frobenius-norm Orthogonal Procrustes projection of $X$ onto the corresponding Stiefel manifold.
\end{proposition}

\begin{proof}
Suppose first that $X$ is tall. Then $X^\top X=V\Sigma^2V^\top$, and hence
\begin{equation}
X(X^\top X)^{-1/2}=(U\Sigma V^\top)(V\Sigma^{-1}V^\top)=UV^\top.
\end{equation}
If $X$ is wide, then $XX^\top=U\Sigma^2U^\top$, which similarly gives
\begin{equation}
(XX^\top)^{-1/2}X=(U\Sigma^{-1}U^\top)(U\Sigma V^\top)=UV^\top.
\end{equation}
Finally, minimizing $\|Q-X\|_F^2$ over the appropriate column- or row-orthogonality constraint is equivalent to maximizing $\operatorname{tr}(Q^\top X)$. The standard Orthogonal Procrustes solution is $Q=UV^\top$. This completes the proof.
\end{proof}

\subsection{Interpretation of Eigenvalue-Weighted Orthogonalization}
\label{sec:eigen_weight_theory}

\begin{proposition}
Let $\bar{E}=[\hat{E}_1|\cdots|\hat{E}_T]$ and $\bar{C}=[\hat{C}_1;\cdots;\hat{C}_T]$ denote the concatenated factors before weighting, and let $\Omega=\operatorname{diag}(w_1,\ldots,w_m)$ collect the positive importance weights of all $m=rT$ retained directions. The polar factors of $\bar{E}\Omega$ and $\Omega\bar{C}$ maximize the weighted alignment between the orthogonalized and original factors:
\begin{equation}
\begin{aligned}
E_{\mathrm{ortho}} &\in \underset{Q\in\mathcal{S}_E}{\operatorname{arg\,max}}\;\sum_{j=1}^{m}w_j\langle q_j,\bar{e}_j\rangle,\\
C_{\mathrm{ortho}} &\in \underset{R\in\mathcal{S}_C}{\operatorname{arg\,max}}\;\sum_{j=1}^{m}w_j\langle r_j,\bar{c}_j\rangle,
\end{aligned}
\end{equation}
where $\bar{e}_j$ and $q_j$ are the corresponding columns of $\bar{E}$ and $Q$, $\bar{c}_j$ and $r_j$ are the corresponding rows of $\bar{C}$ and $R$, and $\mathcal{S}_E$ and $\mathcal{S}_C$ denote the appropriate column- or row-orthogonality constraints. Therefore, directions with larger weights exert greater influence on the Procrustes projections. In ESM, $w_j=\omega_{t,i}$ is monotonic in the ESD eigenvalue $\lambda_{t,i}$, and weighting both factors gives
\begin{equation}
\bar{E}\Omega^2\bar{C}=\sum_{j=1}^{m}w_j^2\bar{e}_j\bar{c}_j,
\end{equation}
with $w_j^2\propto\lambda_{t,i}^{\gamma}$ within each task. This identity describes the product before polar projection: the actual ESM update is formed from the two polar factors, so the proposition establishes how the weights influence the projection objectives rather than an exact per-direction weighting of the final update.
\end{proposition}

\begin{proof}
For any matrix $X$, its polar factor is the solution to the Orthogonal Procrustes problem $\min_{Q\in\mathcal{S}}\|Q-X\|_F^2$. Since $\|Q\|_F^2$ is constant on the corresponding Stiefel manifold, this is equivalent to maximizing $\operatorname{tr}(Q^\top X)$. Substituting $X=\bar{E}\Omega$ gives
\begin{equation}
\operatorname{tr}(Q^\top\bar{E}\Omega)=\sum_{j=1}^{m}w_j\langle q_j,\bar{e}_j\rangle.
\end{equation}
Similarly, substituting $X=\Omega\bar{C}$ gives
\begin{equation}
\operatorname{tr}(R^\top\Omega\bar{C})=\sum_{j=1}^{m}w_j\langle r_j,\bar{c}_j\rangle.
\end{equation}
Hence a larger $w_j$ places a larger coefficient on preserving the alignment of direction $j$ during orthogonalization. Applying $\Omega$ to both factors before projection yields $\bar{E}\Omega^2\bar{C}=\sum_j w_j^2\bar{e}_j\bar{c}_j$. By Eq.~\ref{eq:esm_eigenvalue_weight}, $w_j^2$ is proportional to $\lambda_{t,i}^{\gamma}$ up to a task-wise normalization constant. The final ESM update, however, is the product of the separate polar factors, as defined in Eq.~\ref{eq:esm_reconstruction}; the pre-projection identity above does not imply that the final update is this weighted sum. This completes the proof.
\end{proof}

\section{Method Details}
\subsection{Methodology Pseudocode}
Algorithm~\ref{alg:esm_m} summarizes ESM, the static merging variant of our framework. The colored blocks highlight its three main stages: \colorbox{orange!18}{essential subspace decomposition}, \colorbox{blue!12}{concatenation}, and \colorbox{green!14}{weighted orthogonalization/reconstruction}. ESM first decomposes each task matrix within its essential subspace and truncates the retained components, then concatenates the task-specific factors and applies eigenvalue-weighted orthogonalization before reconstruction.

\newcommand{\algblockheader}[2]{%
\STATE{\setlength{\fboxsep}{2pt}\colorbox{#1}{\parbox{0.92\linewidth}{\centering\textit{#2}}}}}
\newcommand{\algcolorstate}[2]{%
\STATE{\setlength{\fboxsep}{2pt}\colorbox{#1}{\parbox{0.92\linewidth}{#2}}}}

\begin{algorithm}[t]
\caption{ESM}
\label{alg:esm_m}
\begin{algorithmic}[1]
\REQUIRE Task matrices $\{\Delta W_t^{(\ell)}\}_{t=1}^T$ for all layers $\ell \in \mathcal{L}$, pre-trained weights $\{W_0^{(\ell)}\}_{\ell\in\mathcal{L}}$, validation set $\mathcal{D}_{\text{val}}$
\ENSURE ESM merged model parameters $\{W_{\text{ESM}}^{(\ell)}\}_{\ell\in\mathcal{L}}$

\FOR{each task $t = 1$ to $T$ and each layer $\ell \in \mathcal{L}$}
    \algblockheader{orange!18}{\textbf{Essential Subspace Decomposition}}
    \algcolorstate{orange!10}{Obtain the essential basis $E_t^{(\ell)}$ following ESD in Section~\ref{sec:ESD} (Eq.~\ref{eq:activation_shift})}
    \algcolorstate{orange!10}{Project the task matrix onto this basis: $C_t^{(\ell)} \gets (E_t^{(\ell)})^\top \Delta W_t^{(\ell)}$ (Eq.~\ref{eq:esd_factorization})}
    \algcolorstate{orange!10}{Set $r \gets \lfloor d_{\text{out}}/T \rfloor$ and retain $\hat{E}_t^{(\ell)} \gets E_t^{(\ell)}[:,1:r]$, $\hat{C}_t^{(\ell)} \gets C_t^{(\ell)}[1:r,:]$ (Eq.~\ref{eq:esd_approximation})}
    \algcolorstate{orange!10}{Compute the importance weights $\{\omega_{t,i}^{(\ell)}\}_{i=1}^r$ from the retained ESD eigenvalues (Eq.~\ref{eq:esm_eigenvalue_weight})}
\ENDFOR

\FOR{each layer $\ell \in \mathcal{L}$}
    \algblockheader{blue!12}{\textbf{Concatenation}}
    \algcolorstate{blue!8}{Stack truncated bases: $E_{\text{cat}}^{(\ell)} \gets [\hat{E}_1^{(\ell)}, \hat{E}_2^{(\ell)}, \ldots, \hat{E}_T^{(\ell)}]$ (Eq.~\ref{eq:esm_concatenation})}
    \algcolorstate{blue!8}{Stack truncated coordinates: $C_{\text{cat}}^{(\ell)} \gets [\hat{C}_1^{(\ell)}; \hat{C}_2^{(\ell)}; \ldots; \hat{C}_T^{(\ell)}]$ (Eq.~\ref{eq:esm_concatenation})}
    \algcolorstate{blue!8}{Assemble $\Omega^{(\ell)}\gets\operatorname{blkdiag}(\Omega_1^{(\ell)},\ldots,\Omega_T^{(\ell)})$ from the task-wise diagonal weight matrices (Eq.~\ref{eq:esm_weight_matrix})}
    \algblockheader{green!14}{\textbf{Weighted Orthogonalization and Reconstruction}}
    \algcolorstate{green!8}{Weight the concatenated factors: $\tilde{E}_{\text{cat}}^{(\ell)}\gets E_{\text{cat}}^{(\ell)}\Omega^{(\ell)}$, $\tilde{C}_{\text{cat}}^{(\ell)}\gets \Omega^{(\ell)}C_{\text{cat}}^{(\ell)}$ (Eq.~\ref{eq:esm_weighted_factors})}
    \algcolorstate{green!8}{Compute SVDs of $\tilde{E}_{\text{cat}}^{(\ell)}$ and $\tilde{C}_{\text{cat}}^{(\ell)}$ (Eq.~\ref{eq:esm_orthogonalization_svd})}
    \algcolorstate{green!8}{Obtain orthogonal factors $E_{\text{ortho}}^{(\ell)} \gets U_E^{(\ell)}(V_E^{(\ell)})^\top$, $C_{\text{ortho}}^{(\ell)} \gets U_C^{(\ell)}(V_C^{(\ell)})^\top$ (Eq.~\ref{eq:esm_orthogonalization})}
    \algcolorstate{green!8}{Reconstruct the ESM merged update $\Delta W_{\text{ESM}}^{(\ell)} \gets E_{\text{ortho}}^{(\ell)} C_{\text{ortho}}^{(\ell)}$ (Eq.~\ref{eq:esm_reconstruction})}
\ENDFOR

\STATE Select the global coefficient $\alpha^*$ using $\mathcal{D}_{\text{val}}$
\FOR{each layer $\ell \in \mathcal{L}$}
    \STATE Update ESM weights: $W_{\text{ESM}}^{(\ell)} \gets W_0^{(\ell)} + \alpha^* \Delta W_{\text{ESM}}^{(\ell)}$ (Eq.~\ref{eq:esm_final_merge})
\ENDFOR

\STATE \RETURN $\{W_{\text{ESM}}^{(\ell)}\}_{\ell\in\mathcal{L}}$
\end{algorithmic}
\end{algorithm}

Algorithm~\ref{alg:esm_r} summarizes ESM++, the dynamic routing variant. It extracts task-specific residual experts, builds $K_{\mathrm{p}}$ prototypes per task and layer from the ESM-merged model, and selects the highest-scoring expert at each layer during inference.

\begin{algorithm}[t]
\caption{ESM++}
\label{alg:esm_r}
\begin{algorithmic}[1]
\REQUIRE ESM merged weights $\{W_{\text{ESM}}^{(\ell)}\}_{\ell\in\mathcal{L}}$, task-specific weights $\{W_t^{(\ell)}\}_{t=1}^T$, proxy dataset $\mathcal{D}_{\text{proxy}}$, prototype count $K_{\mathrm{p}}$, test input $x$
\ENSURE Output prediction $y$

\FOR{each task $t = 1$ to $T$ and each layer $\ell \in \mathcal{L}$}
    \algblockheader{orange!18}{\textbf{Low-Rank Expert Extraction}}
    \algcolorstate{orange!10}{Compute residual update $\delta W_t^{(\ell)} \gets W_t^{(\ell)} - W_{\text{ESM}}^{(\ell)}$ (Eq.~\ref{eq:esmr_residual})}
    \algcolorstate{orange!10}{Apply ESD to $\delta W_t^{(\ell)}$ and retain low-rank expert factors $(\hat{B}_t^{(\ell)},\hat{A}_t^{(\ell)})$ following Section~\ref{sec:ESM++}}
\ENDFOR

\FOR{each task $t = 1$ to $T$ and each layer $\ell \in \mathcal{L}$}
    \algblockheader{blue!12}{\textbf{Prototype Collection}}
    \algcolorstate{blue!8}{Run $W_{\text{ESM}}$ on task proxy samples and mean-pool each layer input to obtain features $z_{t,i}^{(\ell)}$}
    \algcolorstate{blue!8}{Apply $k$-means to $\{z_{t,i}^{(\ell)}\}_{i=1}^{N}$ to obtain centers $\{p_{t,j}^{(\ell)}\}_{j=1}^{K_{\mathrm{p}}}$ (Eq.~\ref{eq:esmr_prototype})}
\ENDFOR

\STATE Initialize activations with input $x$
\FOR{each layer $\ell \in \mathcal{L}$ during inference}
    \algblockheader{green!14}{\textbf{Prototype-Based Routing and Forward}}
    \algcolorstate{green!8}{Mean-pool current layer input features to obtain $\bar{x}^{(\ell)}$}
    \algcolorstate{green!8}{Compute cosine similarities $s_{t,j}^{(\ell)}$ between $\bar{x}^{(\ell)}$ and every prototype (Eq.~\ref{eq:esmr_routing_score})}
    \algcolorstate{green!8}{Select $(t^\star,j^\star) \gets \operatorname{arg\,max}_{t,j}s_{t,j}^{(\ell)}$ and compose $W_{\text{ESM++}}^{(\ell)} \gets W_{\text{ESM}}^{(\ell)} + \hat{B}_{t^\star}^{(\ell)}\hat{A}_{t^\star}^{(\ell)}$ (Eq.~\ref{eq:esmr_expert_weight})}
    \algcolorstate{green!8}{Perform the layer forward pass using $W_{\text{ESM++}}^{(\ell)}$}
\ENDFOR

\STATE \RETURN prediction $y$
\end{algorithmic}
\end{algorithm}

\subsection{Merging Non-Matrix Parameters}
While most parameters in a transformer are two-dimensional matrices that can be processed by ESM, the network also contains other parameter types. For non-matrix parameters such as bias vectors, layer-normalization parameters, and convolutional-stem parameters, we follow the standard practice in \cite{gargiulo2025task} and apply simple averaging.

\subsection{Target Layers for ESM}
We primarily apply ESM to the linear layers in transformer blocks, including the query, key, value, and output projections in the attention module, as well as the up- and down-projection layers in the MLP. Based on the eigenvalue distribution of the output shifts across these layers, we select the query, key, value, and MLP up-projection layers as the target layers for ESM merging and ESM++ routing in vision models. The remaining layers are merged by simply averaging the corresponding parameters across all fine-tuned models. For language models, we apply ESM to all linear layers in the transformer blocks.

\subsection{Implementation-Level Acceleration of Prototype Routing}
Although ESM++ selects only one expert per layer, a naive implementation incurs substantial overhead from repeated pooling, prototype scoring, Python dispatch, and materializing or transferring selected dense weights. We accelerate the implementation while retaining the same prototype-based score and argmax rule, up to floating-point effects. For ViT, the low-rank expert factors remain on the GPU and are applied directly, using $x(BA)^\top=(xA^\top)B^\top$, followed by \texttt{torch.compile} with the \texttt{reduce-overhead} mode. For RoBERTa, per-layer monkey-patched closures are replaced with a shared routed linear module; the encoder is compiled while task-specific classification heads remain eager. The factorized expert path also avoids constructing a separate dense expert matrix for each sample in a batch.

For ViT-B/32, the optimized path reduces single-sample latency from 11.75\,ms to 2.24\,ms (5.3$\times$ faster), and brings batch-128 latency close to the base-model latency. For RoBERTa, the shared routed module with dense experts reaches 0.68$\times$ base latency at $B=1$ (9.84\,ms to 2.31\,ms). At larger batches, factorized experts reduce overhead from 1.38$\times$ to 1.18$\times$ base latency at $B=128$; at $B=1$, this factorized path is approximately at base-model latency. In controlled evaluations with real prototypes, the ViT optimized path retained its 92.09\% aggregate accuracy, and the RoBERTa structured dense and factorized paths matched each other (96.24\% average accuracy). However, the legacy monkey-patched RoBERTa evaluator yielded 95.72\%, illustrating that near-tied prototype scores can make layerwise argmax decisions sensitive to floating-point operation order. We therefore keep the routing implementation fixed when comparing accuracy.

\section{Experiment Details}

\subsection{Generative Language Model Evaluation Datasets}
\label{sec:generative_lm_eval_datasets}

Following the five-category organization of MergeBench \cite{he2025mergebench}, we evaluate on the datasets and metrics listed in Table~\ref{tab:mergebench_eval_datasets}.

\begin{table*}[t]
\centering
\caption{Datasets and metrics used for generative language model evaluation. The five-category organization follows MergeBench \cite{he2025mergebench}.}
\label{tab:mergebench_eval_datasets}
\setlength{\tabcolsep}{6pt}
\begin{tabular}{llll}
\toprule[1.5pt]
\textbf{Category} & \textbf{Dataset} & \textbf{Metric} & \textbf{\# Data} \\
\midrule[0.5pt]
Instruction-following & IFEval \cite{zhou2023instruction} & Prompt-Level Loose Accuracy, Inst-Level Loose Accuracy & 541 \\
\midrule[0.5pt]
Mathematics & GSM8K \cite{cobbe2021training} & Exact match (8-shot CoT; flexible extraction) & 1320 \\
\midrule[0.5pt]
\multirow{3}{*}{Multilingual understanding} & M\_MMLU \cite{lai2023okapi} & Accuracy & 60K \\
& M\_ARC \cite{lai2023okapi} & Normalized accuracy & 10.34K \\
& M\_Hellaswag \cite{lai2023okapi} & Normalized accuracy & 37.35K \\
\midrule[0.5pt]
\multirow{2}{*}{Coding} & Humaneval+ \cite{chen2021evaluating} & Pass@1 & 164 \\
& MBPP+ \cite{austin2021program} & Pass@1 & 378 \\
\midrule[0.5pt]
\multirow{4}{*}{Safety} & WildGuardTest \cite{han2024wildguard} & Response Harmfulness ($1-\mathrm{harm}$) & 1730 \\
& HarmBench \cite{mazeika2024harmbench} & $1-\mathrm{ASR}$ (Attack Success Rate) & 410 \\
& DoAnythingNow \cite{shen2024anything} & RTA (Refuse To Answer, $1-\mathrm{compliance}$) & 15.14K \\
& XSTest \cite{rottger2024xstest} & Accuracy & 450 \\
\bottomrule[1.5pt]
\end{tabular}
\end{table*}

\subsection{Calculation of Energy Retention}
\label{sec:energy_retention}
\figpanel{fig:esd_vs_svd}{a} shows the cumulative energy retained when preserving different proportions of components.
For the SVD-based method, energy retention is calculated as the ratio of the sum of squares of the retained singular values to the sum of squares of all singular values. For our ESD method, it is defined as the ratio of the sum of the retained eigenvalues to the sum of all eigenvalues. These spectral quantities measure the energy carried by each component.

\subsection{Subspace Similarity Metrics for Proxy Size Analysis}
\label{sec:proxy_subspace_metrics}

We evaluate how well a subspace estimated from a limited proxy set matches a reference subspace estimated from the full test set. Let the reference subspace be $\mathcal{U}_{\text{test}}$ with orthonormal basis $U_{\text{test}} \in \mathbb{R}^{d \times r}$ and eigenvalues $\lambda_1 \geq \cdots \geq \lambda_r > 0$, and let the proxy-estimated subspace $\mathcal{U}_s$ have basis $U_s \in \mathbb{R}^{d \times r}$. We form $C=U_s^\top U_{\text{test}}$, whose singular values satisfy $\sigma_i=\cos\theta_i$, where $\theta_i$ are the principal angles. Comparing subspaces this way avoids requiring their individual basis vectors to match; values closer to one indicate stronger alignment.

\paragraph{Equal-weight projection similarity.}
This metric averages the squared cosines of all principal angles:
\begin{equation}
\mathrm{PS}_{\mathrm{eq}} = \frac{1}{r} \sum_{i=1}^{r} \sigma_i^2 = \frac{\|U_s^\top U_{\text{test}}\|_F^2}{r}.
\end{equation}
This basis-invariant score measures average overlap and ranges from zero to one, reaching one when the subspaces coincide. Unlike the weighted score below, it treats all retained directions equally, regardless of their output-shift energy.

\paragraph{Eigenvalue-weighted projection similarity.}
Since different principal directions contribute different amounts of output-shift energy, we also compute an eigenvalue-weighted score:
\begin{equation}
\mathrm{PS}_{\mathrm{w}} = \frac{\sum_{i=1}^{r} \lambda_i \|U_s^\top u_i\|_2^2}{\sum_{i=1}^{r} \lambda_i},
\end{equation}
where $u_i$ is the $i$-th column of $U_{\text{test}}$. The projection term $\|U_s^\top u_i\|_2^2$ measures how much of each reference direction is captured by the proxy, while $\lambda_i$ weights that alignment by output-shift energy. Thus, $\mathrm{PS}_{\mathrm{w}}$ measures the fraction of reference-subspace energy preserved, emphasizing dominant directions. It complements the equal-weight score when low-energy tail directions are less important for retaining functional information.

\paragraph{Maximum principal angle.}
We further report the worst-case subspace misalignment:
\begin{equation}
\theta_{\max} = \arccos(\sigma_r).
\end{equation}
This metric is conservative because it is determined by the least-aligned direction, and it can reveal poorly estimated tail components. Unlike $\mathrm{PS}_{\mathrm{w}}$, it does not account for energy, so a large maximum angle may coexist with strong retention of dominant directions.

\begin{figure*}[t]
    \centering
    \includegraphics[width=\linewidth]{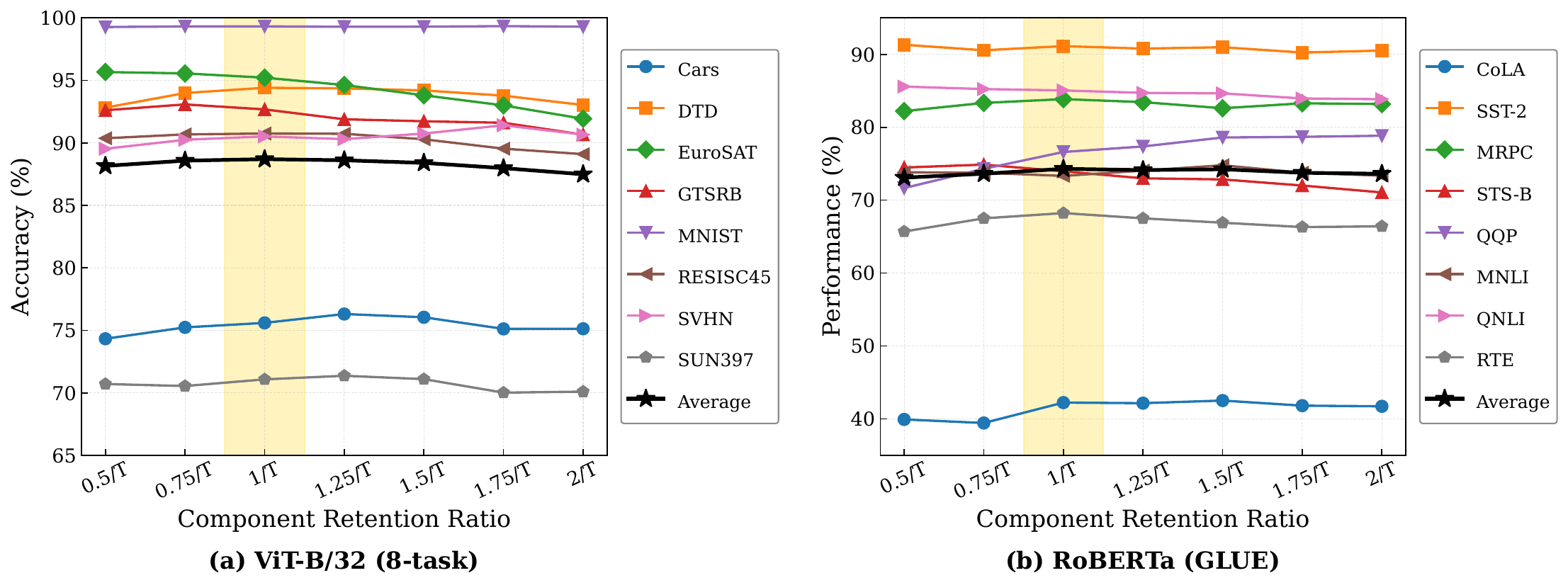}
    \caption{Ablation study on the impact of component retention ratio on merged model performance. $T$ denotes the number of tasks.}
    \label{fig:suppl_ablation_rank_r}
\end{figure*}

\subsection{Ablation Study of the ESM Rank Budget}
Our default assigns each task a rank budget of $r = \lfloor d_{\text{out}} / T \rfloor$, where $T$ is the number of tasks and $d_{\text{out}}$ is the output dimension. This allocation gives each task a compact representation while keeping the total budget on the order of $d_{\text{out}}$. Fig.~\ref{fig:suppl_ablation_rank_r} varies this budget from $\lfloor 0.5\,d_{\text{out}} / T \rfloor$ to $\lfloor 2.0\,d_{\text{out}} / T \rfloor$. Performance remains comparable across this range: smaller budgets test whether dominant task directions suffice, while larger budgets retain more residual structure, including lower-energy components. This stability is consistent with task-relevant energy being concentrated in a few dominant directions and with eigenvalue weighting and orthogonalization reducing the influence of low-energy components. The default is a practical allocation rather than a uniquely optimal rank.

\begin{figure*}[t]
    \centering
    \includegraphics[width=\linewidth]{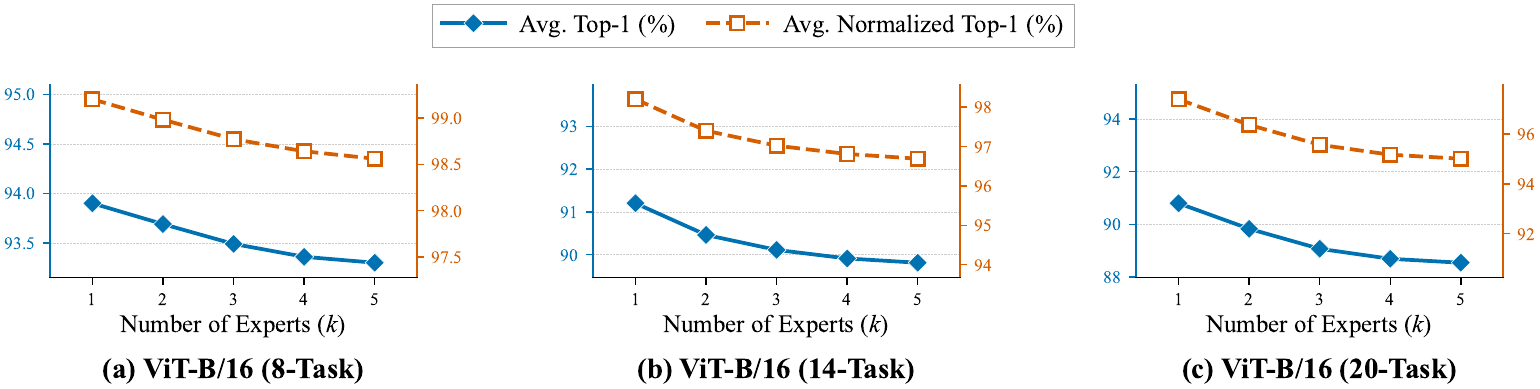}
    \caption{Effect of the number of selected experts in ESM++ ($r=32$). Routing to a single expert yields the best multi-task performance.}
    \label{fig:topk_moe}
\end{figure*}

\begin{table*}[t]
\centering
\caption{Single-prototype ESM++ results with prototypes extracted from the ESM-merged base, compared with oracle routing. Routing accuracy is averaged over the nine vision settings; normalized accuracy is shown in parentheses.}
\label{tab:oracle_routing}
\setlength{\tabcolsep}{3.9pt}
\resizebox{\textwidth}{!}{%
\begin{tabular}{lccccccccccc}
\toprule[1.5pt]
\multirow{2}{*}{Method} & \multirow{2}{*}{\makecell{Routing\\Strategy}} & \multirow{2}{*}{\makecell{Routing\\Accuracy}} & \multicolumn{3}{c}{ViT-B/32} & \multicolumn{3}{c}{ViT-B/16} & \multicolumn{3}{c}{ViT-L/14}
\\
\cmidrule[0.5pt](lr){4-6}\cmidrule[0.5pt](lr){7-9}\cmidrule[0.5pt](lr){10-12} &&& 8 tasks & 14 tasks & 20 tasks & 8 tasks & 14 tasks & 20 tasks & 8 tasks & 14 tasks & 20 tasks
\\
\midrule[0.5pt]

\rowcolor[HTML]{ECECEC}\textit{Pre-trained} & -- & -- & 48.3 & 57.2 & 56.1 & 55.3 & 61.3 & 59.7 & 64.7 & 68.2 & 65.2 \\
\rowcolor[HTML]{ECECEC}\textit{Fine-tuned} & -- & -- & 92.8 & 90.9 & 91.3 & 94.6 & 92.8 & 93.2 & 95.8 & 94.3 & 94.7 \\
\rowcolor[HTML]{ECECEC}ESM & -- & -- & 88.6$_{(95.4)}$ & 83.9$_{(92.4)}$ & 82.3$_{(90.1)}$ & 91.6$_{(96.7)}$ & 87.6$_{(94.4)}$ & 85.3$_{(91.6)}$ & 94.7$_{(98.8)}$ & 91.3$_{(96.8)}$ & 90.7$_{(95.7)}$ \\

\midrule[0.5pt]

ESM++ ($r=8$) & Prototype & 84.7\% & 91.5$_{(98.5)}$ & 88.0$_{(96.8)}$ & 87.4$_{(95.6)}$ & \textbf{93.6}$_{(\textbf{98.8})}$ & \textbf{90.5}$_{(\textbf{97.5})}$ & 89.6$_{(96.1)}$ & 95.5$_{(\textbf{99.7})}$ & 92.8$_{(98.4)}$ & 92.9$_{(98.0)}$ \\

ESM++ ($r=8$) & Oracle & 100.0\% & \textbf{91.7}$_{(\textbf{98.8})}$ & \textbf{88.2}$_{(\textbf{97.0})}$ & \textbf{88.0}$_{(\textbf{96.2})}$ & 93.5$_{(98.7)}$ & \textbf{90.5}$_{(97.4)}$ & \textbf{90.2}$_{(\textbf{96.8})}$ & \textbf{95.6}$_{(\textbf{99.7})}$ & \textbf{93.0}$_{(\textbf{98.5})}$ & \textbf{93.2}$_{(\textbf{98.3})}$ \\

\midrule[0.5pt]

ESM++ ($r=32$) & Prototype & 84.5\% & 92.1$_{(99.2)}$ & 89.0$_{(97.8)}$ & 88.6$_{(96.8)}$ & 93.9$_{(99.2)}$ & 91.2$_{(98.2)}$ & 90.8$_{(97.4)}$ & 95.7$_{(99.9)}$ & 93.3$_{(98.9)}$ & 93.5$_{(98.6)}$ \\
ESM++ ($r=32$) & Oracle & 100.0\% & \textbf{92.4}$_{(\textbf{99.5})}$ & \textbf{89.2}$_{(\textbf{98.0})}$ & \textbf{89.2}$_{(\textbf{97.5})}$ & \textbf{94.0}$_{(\textbf{99.3})}$ & \textbf{91.3}$_{(\textbf{98.3})}$ & \textbf{91.3}$_{(\textbf{97.9})}$ & \textbf{95.9}$_{(\textbf{100.0})}$ & \textbf{93.5}$_{(\textbf{99.1})}$ & \textbf{93.8}$_{(\textbf{99.0})}$ \\

\bottomrule[1.5pt]
\end{tabular}%
}
\end{table*}

\begin{table*}[!t]
\centering
\caption{Per-setting routing accuracy of the single-prototype ESM++ router using prototypes from the ESM-merged base, across the nine vision benchmark settings in Table~\ref{tab:oracle_routing}.}
\label{tab:routing_accuracy_settings}
\setlength{\tabcolsep}{2.9pt}
\resizebox{\textwidth}{!}{%
\begin{tabular*}{\textwidth}{@{\extracolsep{\fill}}lccccccccc}
\toprule[1.5pt]
\multirow{2}{*}{Expert Rank} & \multicolumn{3}{c}{ViT-B/32} & \multicolumn{3}{c}{ViT-B/16} & \multicolumn{3}{c}{ViT-L/14} \\
\cmidrule(lr){2-4}\cmidrule(lr){5-7}\cmidrule(lr){8-10}
& 8 tasks & 14 tasks & 20 tasks & 8 tasks & 14 tasks & 20 tasks & 8 tasks & 14 tasks & 20 tasks \\
\midrule[0.5pt]
$r=8$ & 88.2\% & 87.6\% & 77.1\% & 89.8\% & 88.7\% & 78.5\% & 88.1\% & 87.3\% & 76.8\% \\
$r=32$ & 88.1\% & 87.4\% & 76.8\% & 89.6\% & 88.6\% & 78.3\% & 88.1\% & 87.2\% & 76.6\% \\
\bottomrule[1.5pt]
\end{tabular*}
}
\end{table*}

\subsection{Effect of the Number of Routed Experts}

Fig.~\ref{fig:topk_moe} analyzes the effect of the number of selected experts in ESM++ ($r=32$). The results show that routing to a single expert achieves the best performance. As more experts are selected, the additional task-specific residuals can introduce interference among tasks, which degrades the overall multi-task performance.
This demonstrates that the proposed training-free routing mechanism can achieve efficient and high-performance inference by selecting only one expert.

\subsection{Prototype-Based and Oracle Routing}
\label{sec:prototype_oracle_routing}

We compare prototype-based routing, which selects experts according to their similarity to task prototypes, with oracle routing, which uses the ground-truth task identity. This evaluation separates the effect of routing errors from the capacity of the retained principal components. As shown in Table~\ref{tab:oracle_routing}, prototype-based routing requires no additional learned routing network. Oracle routing serves as a reference for how much task-specific performance the low-rank residual experts preserve when the task identity is known. Because ESD minimizes expected output truncation error under the proxy distribution, retaining a small rank of $r=8$ is already effective. This rank is tiny compared with the hidden dimensions of ViT-B ($768$) and ViT-L/14 ($1024$); with oracle routing, normalized accuracy ranges from $96.2\%$ to $100.0\%$ across the nine settings.

\subsection{Per-Setting Routing Accuracy}

Table~\ref{tab:routing_accuracy_settings} provides routing accuracy by backbone, task count, and residual rank.
Routing accuracy is highest in the 8-task settings (88.1--89.8\%) and remains strong in the 14-task settings (87.2--88.7\%), but drops to 76.6--78.5\% in the 20-task settings. Averaged across all nine settings, accuracy is 84.7\% for $r=8$ and 84.5\% for $r=32$.

\begin{figure*}[t]
    \centering
    \includegraphics[width=\linewidth]{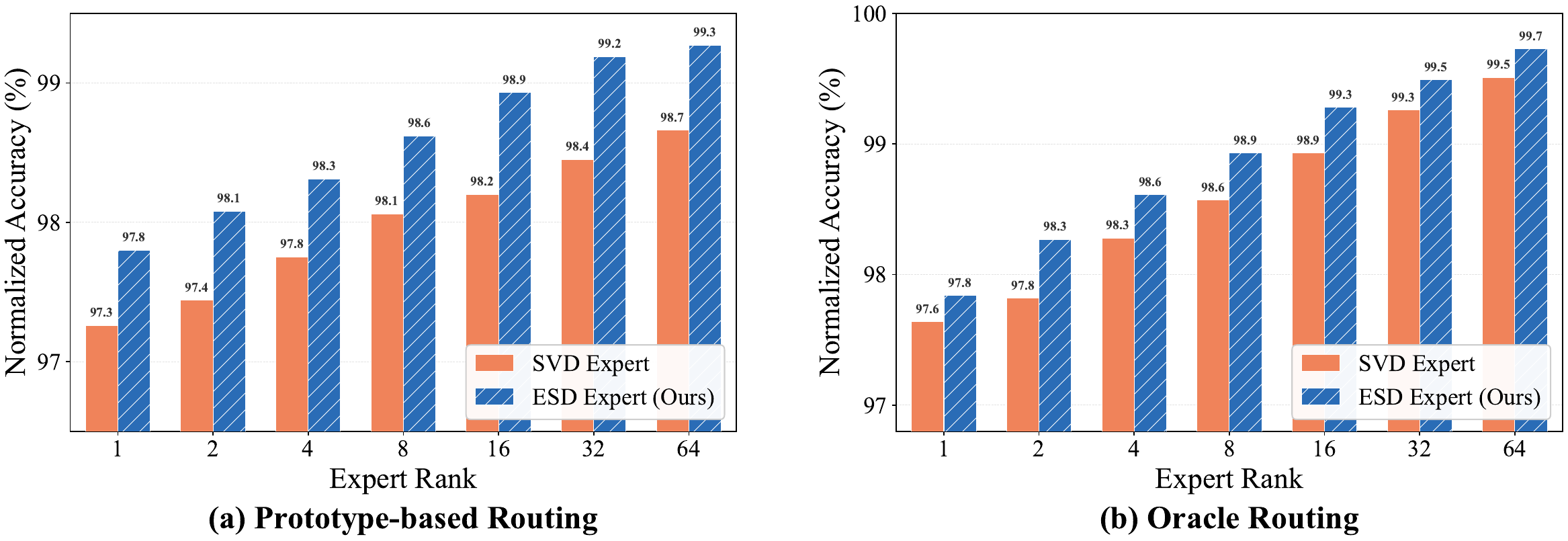}
    \caption{Comparison of low-rank expert construction methods in ESM++. We compare experts obtained by directly applying SVD to parameter update matrices with those obtained by the proposed ESD. Panels (a) and (b) report normalized accuracy under prototype-based routing and oracle routing, respectively, where the x-axis denotes the retained rank and the y-axis denotes normalized accuracy.}
    \label{fig:fused_model_comparison}
\end{figure*}

\begin{figure*}[t]
\centering
\includegraphics[width=\linewidth]{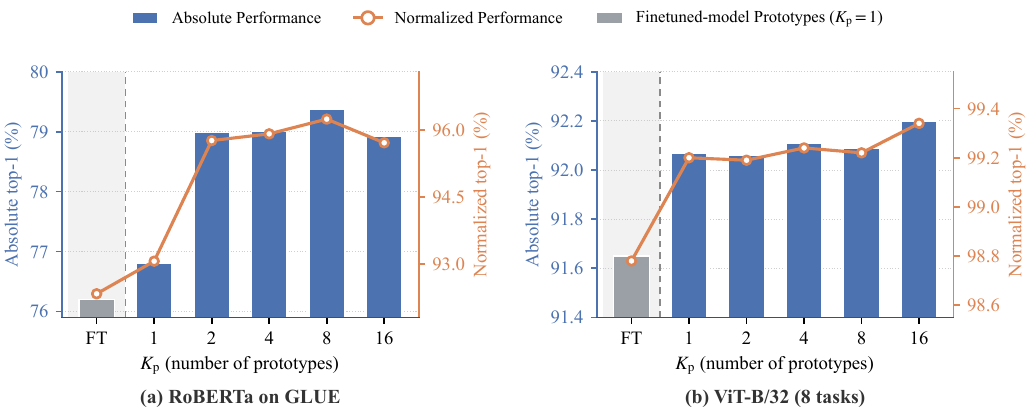}
\caption{Prototype-source and prototype-count ablations on (a) RoBERTa (GLUE) and (b) ViT-B/32 (8 tasks). Bars show absolute top-1 accuracy (left axes), and lines show normalized top-1 accuracy (right axes). The leftmost ``FT'' setting uses one prototype sampled from features extracted by the corresponding fine-tuned model for each task. The remaining settings use prototypes sampled from features extracted by the ESM-merged model, with $K_{\mathrm{p}}$ varied.}
\label{fig:proto_ablation}
\end{figure*}

\begin{table*}[t]
\centering
\caption{Global scaling coefficient $\alpha$ selected on the validation set for each benchmark setting.}
\label{tab:alpha_selection}
\setlength{\tabcolsep}{4.2pt}
\begin{tabular}{ccccccccccc}
\toprule[1.5pt]
\multicolumn{3}{c}{\textbf{ViT-B/32}} & \multicolumn{3}{c}{\textbf{ViT-B/16}} & \multicolumn{3}{c}{\textbf{ViT-L/14}} & \textbf{RoBERTa} & \textbf{Llama-3.2-3B} \\
\cmidrule[0.5pt](lr){1-3}\cmidrule[0.5pt](lr){4-6}\cmidrule[0.5pt](lr){7-9}\cmidrule[0.5pt](lr){10-10}\cmidrule[0.5pt](lr){11-11}
8-task & 14-task & 20-task & 8-task & 14-task & 20-task & 8-task & 14-task & 20-task & GLUE & MergeBench \\
\midrule[0.5pt]
\textit{0.76} & \textit{0.69} & \textit{0.65} & \textit{0.76} & \textit{0.67} & \textit{0.67} & \textit{0.83} & \textit{0.75} & \textit{0.67} & \textit{2.80} & \textit{0.90} \\
\bottomrule[1.5pt]
\end{tabular}
\end{table*}

\subsection{Comparison of Low-Rank Expert Construction Methods}

Fig.~\ref{fig:fused_model_comparison} compares two ways to construct low-rank residual experts for ESM++: direct SVD on parameter updates and the proposed ESD. Evaluating both methods across the same retained ranks highlights the effect of the decomposition on expert quality at a given rank. ESD consistently achieves higher normalized accuracy under both prototype-based and oracle routing. Notably, the performance gap over SVD is larger under prototype-based routing than under oracle routing. This suggests that expert quality and routing are complementary: the stronger ESD experts make prototype-based routing more effective in terms of final accuracy, further widening the end-task advantage over SVD. The gain under oracle routing also confirms that ESD preserves more task-relevant behavior even when the task identity is known.

\subsection{Prototype Source and Count Ablations}
\label{sec:prototype_ablation}

Fig.~\ref{fig:proto_ablation} compares prototype sources and counts at residual rank $r=8$ for RoBERTa and $r=32$ for ViT-B/32. For both RoBERTa and ViT-B/32, using prototypes derived from ESM-merged-model features improves performance over prototypes derived from the corresponding fine-tuned-model features. The models differ in their sensitivity to prototype count: RoBERTa performs best at $K_{\mathrm{p}}=8$, with a slight drop at $K_{\mathrm{p}}=16$, whereas ViT-B/32 is largely insensitive, with absolute accuracy varying by less than $0.2$ points across $K_{\mathrm{p}}\in\{1,2,4,8,16\}$. Thus, changing the prototype source benefits both settings, while increasing the count mainly benefits RoBERTa. We therefore use $K_{\mathrm{p}}=8$ for language tasks and $K_{\mathrm{p}}=1$ for vision tasks.

\subsection{Selection of the Global Scaling Coefficient}

We report the global scaling coefficient $\alpha$ selected on the validation set, as shown in Table \ref{tab:alpha_selection}. Based on the empirical ranges used in previous model merging studies \cite{gargiulo2025task, marczakno}, we set the search interval for $\alpha$ between 0.0 and 5.0 and use ternary search to select the coefficient. The selected $\alpha$ decreases as the number of tasks increases, likely because merging more tasks amplifies the norm of the combined updates.

\subsection{Leave-One-Out OOD Generalization}
\label{sec:ood_results}

We evaluate OOD routing on ViT-B/32 using leave-one-task-out folds: each fold merges the remaining task experts and evaluates on the held-out task, for which ESM++ has no dedicated residual expert. This experiment tests whether ESM++ can route held-out-task inputs to suitable experts from the available task pool and thereby improve OOD accuracy over the static ESM baseline. We report ESM and ESM++ results. Normalized accuracy is computed per held-out task relative to its fine-tuned expert and then averaged. For the 8-task benchmark, $\alpha$ is selected on in-distribution validation data within each fold; for the 14- and 20-task benchmarks, we use the shared full-set values $\alpha=0.6888$ and $\alpha=0.6525$, respectively.

Table~\ref{tab:ood_k} shows that increasing the prototype count from $K_{\mathrm{p}}=1$ to $K_{\mathrm{p}}=2$ improves average OOD accuracy by $0.84$ points and normalized accuracy by $0.90$ points. Further increases largely saturate, with $K_{\mathrm{p}}=16$ giving the highest averages in this sweep. Thus, a modest increase in prototype count improves OOD performance, but the gain is not monotonic with $K_{\mathrm{p}}$.

\begin{table*}[t]
\centering
\caption{Per-task leave-one-out OOD results on ViT-B/32 for the eight-task benchmark. Entries report absolute top-1 accuracy (\%), with normalized accuracy in parentheses. The best and second-best entries in each row are boldfaced and underlined, respectively; ties share the same style.}
\label{tab:ood_k}
\small
\setlength{\tabcolsep}{5pt}
\resizebox{\textwidth}{!}{%
\begin{tabular*}{\textwidth}{@{\extracolsep{\fill}}lcccccc}
\toprule[1.5pt]
\multirow{2}{*}{Held-out task} & \multirow{2}{*}{ESM} & \multicolumn{5}{c}{ESM++} \\
\cmidrule(lr){3-7}
& & $K_{\mathrm{p}}=1$ & $K_{\mathrm{p}}=2$ & $K_{\mathrm{p}}=4$ & $K_{\mathrm{p}}=8$ & $K_{\mathrm{p}}=16$ \\
\midrule[0.5pt]
Cars & 46.18 (59.06) & 48.99 (62.65) & \underline{49.12 (62.83)} & 49.02 (62.70) & 48.75 (62.35) & \textbf{49.92 (63.85)} \\
DTD & 36.65 (37.40) & \underline{40.05 (40.88)} & 39.84 (40.66) & \underline{40.05 (40.88)} & \textbf{40.11 (40.93)} & 40.00 (40.83) \\
EuroSAT & 41.56 (42.01) & 42.56 (43.02) & \underline{42.78 (43.24)} & \textbf{43.56 (44.03)} & \textbf{43.56 (44.03)} & \textbf{43.56 (44.03)} \\
GTSRB & 35.77 (36.08) & 39.32 (39.66) & \textbf{40.25 (40.60)} & 39.29 (39.63) & \underline{40.06 (40.40)} & 38.23 (38.56) \\
MNIST & \textbf{86.68 (86.95)} & 85.52 (85.79) & 85.55 (85.82) & 85.98 (86.25) & 85.87 (86.14) & \underline{86.14 (86.41)} \\
RESISC45 & 42.56 (44.42) & 44.79 (46.76) & \underline{45.02 (46.99)} & 44.94 (46.91) & 44.94 (46.91) & \textbf{45.22 (47.21)} \\
SVHN & 55.45 (56.83) & 52.55 (53.86) & \underline{56.97 (58.39)} & 56.04 (57.44) & 55.87 (57.27) & \textbf{57.49 (58.93)} \\
SUN397 & 48.61 (64.52) & 53.80 (71.41) & \underline{54.76 (72.68)} & \textbf{54.99 (72.98)} & 54.69 (72.58) & 54.71 (72.61) \\
\midrule
Average & 49.18 (53.41) & 50.95 (55.50) & \underline{51.79 (56.40)} & 51.73 (56.35) & 51.73 (56.33) & \textbf{51.91 (56.55)} \\
\bottomrule[1.5pt]
\end{tabular*}
}
\end{table*}

Table~\ref{tab:ood_scale} summarizes average OOD performance for the 14- and 20-task settings. For both task-set sizes, ESM++ outperforms ESM in absolute and normalized accuracy, and $K_{\mathrm{p}}=16$ provides a further gain over $K_{\mathrm{p}}=1$.

\begin{table*}[!t]
\centering
\caption{Average leave-one-task-out OOD accuracy on ViT-B/32 for the 14- and 20-task settings. Entries report absolute top-1 accuracy (\%) followed by normalized accuracy in parentheses.}
\label{tab:ood_scale}
\small
\setlength{\tabcolsep}{4pt}
\resizebox{\textwidth}{!}{%
\begin{tabular*}{\textwidth}{@{\extracolsep{\fill}}lcccccc}
\toprule[1.5pt]
\multirow{2}{*}{Statistic} & \multicolumn{3}{c}{14-task setting} & \multicolumn{3}{c}{20-task setting} \\
\cmidrule(lr){2-4}\cmidrule(lr){5-7}
& ESM & ESM++ ($K_{\mathrm{p}}=1$) & ESM++ ($K_{\mathrm{p}}=16$) & ESM & ESM++ ($K_{\mathrm{p}}=1$) & ESM++ ($K_{\mathrm{p}}=16$) \\
\midrule[0.5pt]
Average & 52.31 (57.40) & 54.34 (59.74) & \textbf{54.94 (60.36)} & 56.47 (61.59) & 58.83 (64.25) & \textbf{59.57 (65.04)} \\
\bottomrule[1.5pt]
\end{tabular*}
}
\end{table*}

\subsection{Performance on Individual Tasks}
Fig.~\ref{fig:radar_comparison} provides the detailed per-task results of CLIP model merging across different backbones, complementing the average performance reported in the main text. The results show how ESM and ESM++ perform on each individual task.

\begin{figure}[t]
    \centering
    \subfloat[ViT-B/32]{%
        \includegraphics[width=0.985\linewidth]{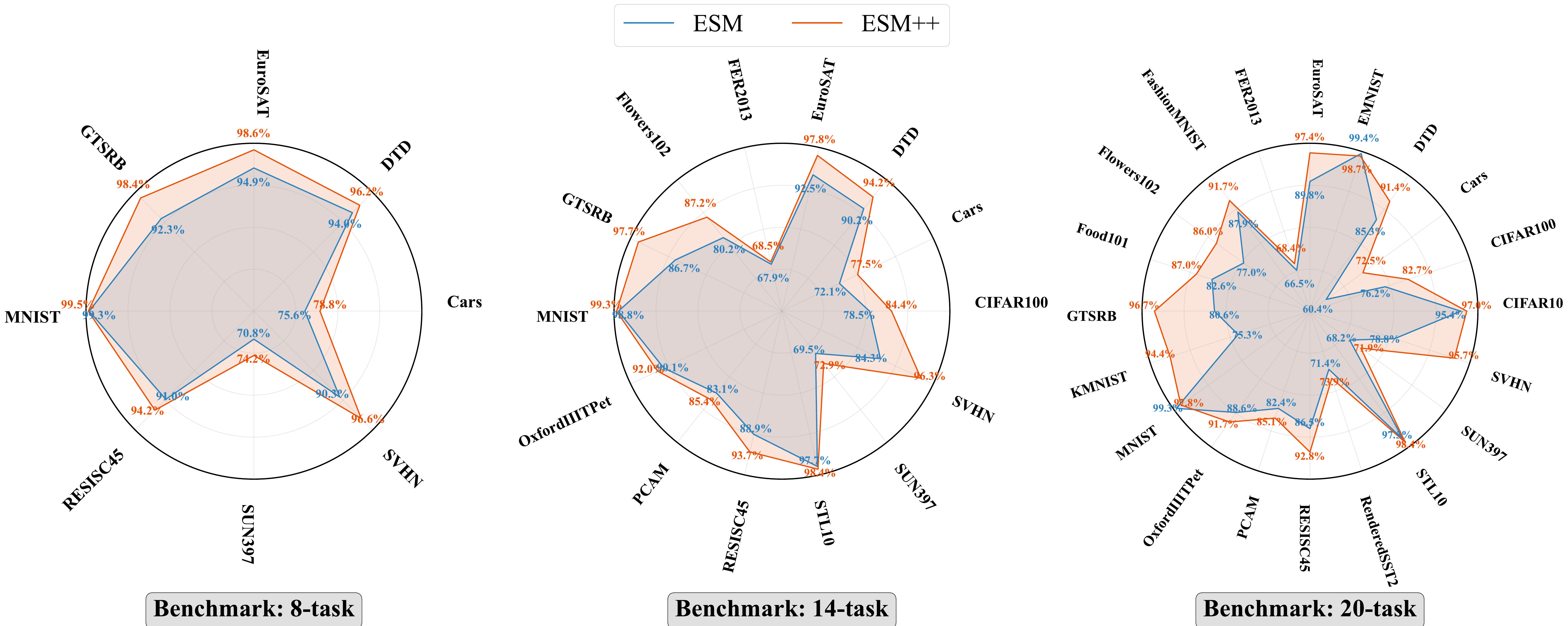}}
    

    \subfloat[ViT-B/16]{%
        \includegraphics[width=0.985\linewidth]{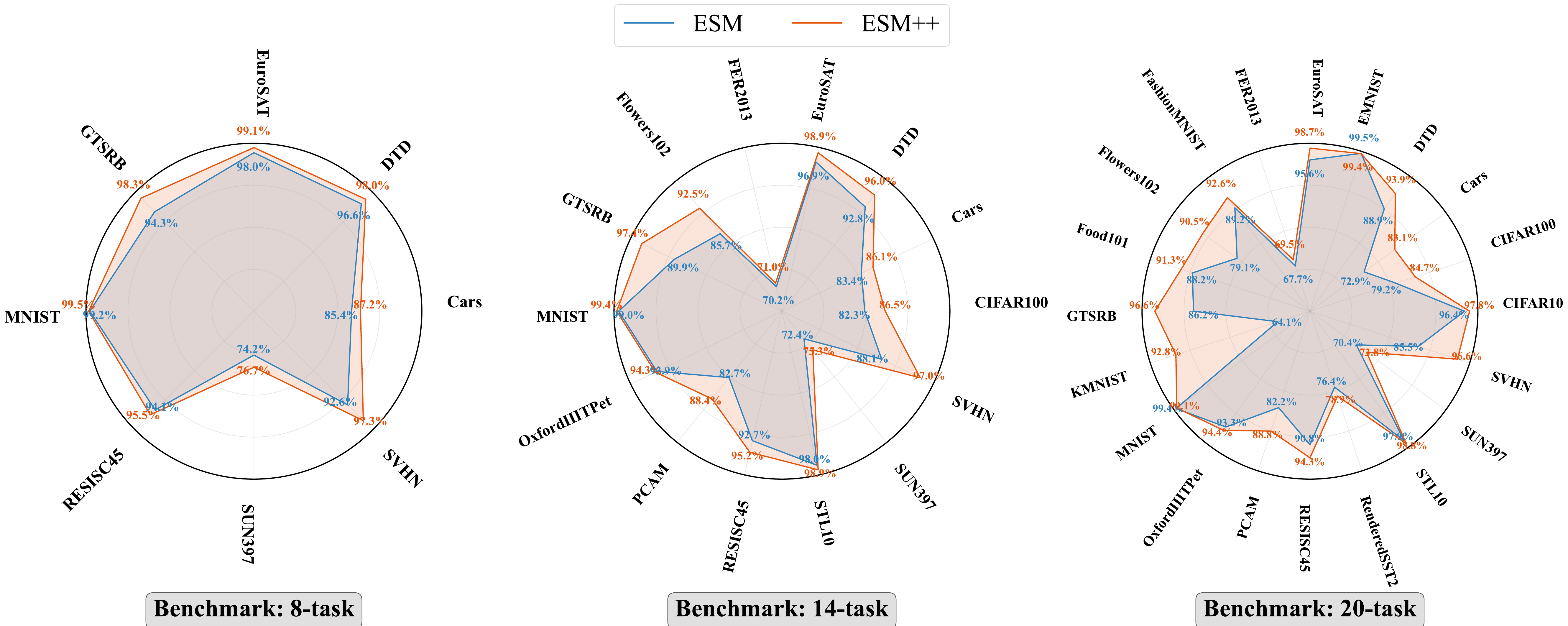}}


    \subfloat[ViT-L/14]{%
        \includegraphics[width=0.985\linewidth]{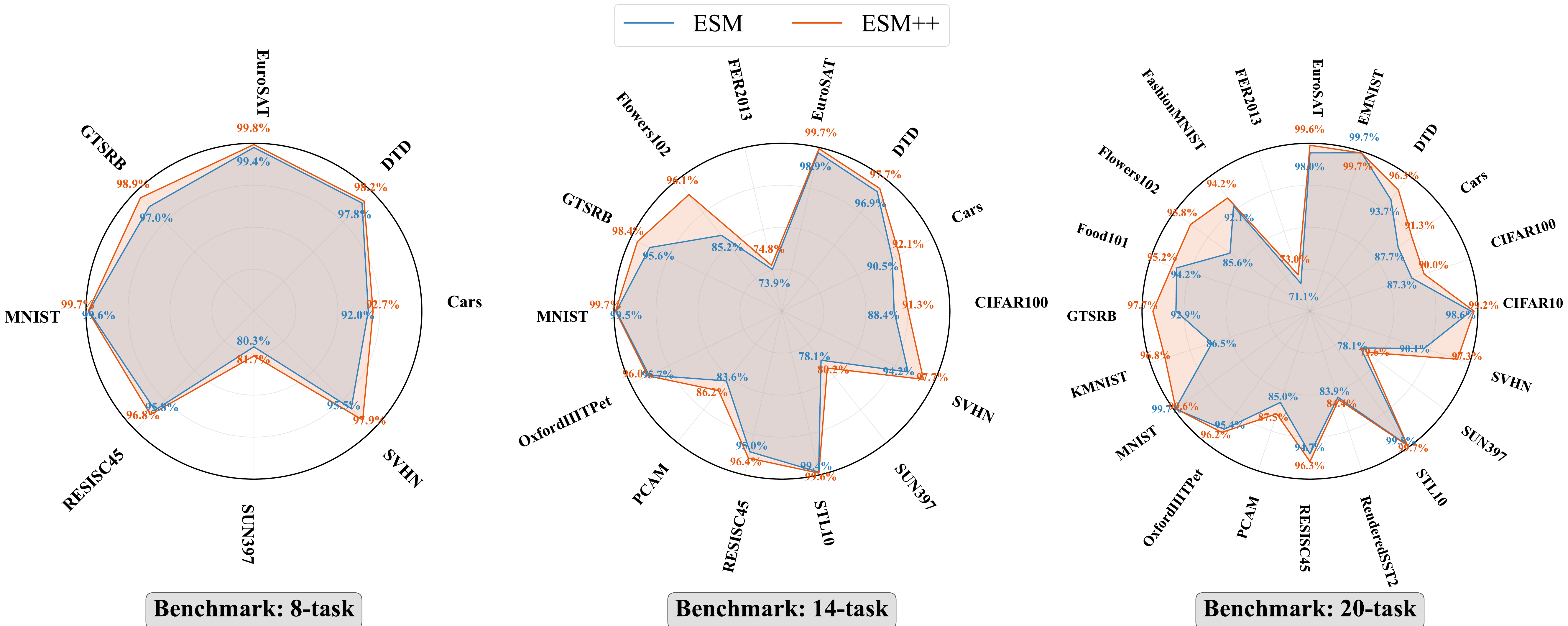}}
    
    \caption{Per-task CLIP merging performance of ESM and ESM++ ($r=32$) across three backbones.}
    \label{fig:radar_comparison}
\end{figure}

\end{document}